\documentclass[conference]{IEEEtran}
\usepackage{times} 

\usepackage{amsmath, amssymb, amsfonts, amsthm, bm, bbm, mathtools}

\usepackage[numbers, sort&compress]{natbib} 
\usepackage{graphicx}
\usepackage{subcaption}
\usepackage{multicol}
\usepackage{hhline}
\usepackage{booktabs}
\usepackage{subcaption}
\usepackage[bookmarks=true]{hyperref}
\usepackage{cleveref}
\usepackage{algorithm}
\usepackage{algorithmic}

\newtheorem*{VI}{Ville's Inequality \cite{Ville_1939}}
\newtheorem{thm}{Theorem}

\newtheorem{lem}{Lemma}
\usepackage{xcolor}

\theoremstyle{definition}
\newtheorem{defn}{Definition}
\newtheorem{rem}{Remark}

\newcommand{\OurMethod}{N-SCORE}

\begin{document}

\title{Beyond Binary Success: Sample-Efficient and Statistically Rigorous Robot Policy Comparison}

\author{\authorblockN{David Snyder$^{1}$, Apurva Badithela$^{3}$, Nikolai Matni$^{1}$, \\ George Pappas$^{1}$, Anirudha Majumdar$^{3}$, Masha Itkina$^{2\ddagger}$, and Haruki Nishimura$^{2\ddagger}$}
\authorblockA{$^{1}$University of Pennsylvania, $^{2}$Toyota Research Institute (TRI), $^{3}$Princeton University. \\ $^\ddagger$Equal Advising. Corresponding Author: \texttt{dsnyder5@seas.upenn.edu}}
}

\maketitle

\begin{abstract}
Generalist robot manipulation policies are becoming increasingly capable, but are limited in evaluation to a small number of hardware rollouts. This strong resource constraint in real-world testing necessitates both more informative performance measures and reliable and efficient evaluation procedures to properly assess model capabilities and benchmark progress in the field. This work presents a novel framework for robot policy comparison that is sample-efficient, statistically rigorous, and applicable to a broad set of evaluation metrics used in practice. Based on safe, anytime-valid inference (SAVI), our test procedure is \emph{sequential}, allowing the evaluator to \emph{stop early} when sufficient statistical evidence has accumulated to reach a decision at a pre-specified level of confidence. Unlike previous work developed for binary success, our unified approach addresses a wide range of informative metrics: from discrete partial credit task progress to continuous measures of episodic reward or trajectory smoothness, spanning both parametric and nonparametric comparison problems. Through extensive validation on simulated and real-world evaluation data, we demonstrate up to 70\% reduction in evaluation burden compared to standard batch methods and up to 50\% reduction compared to state-of-the-art sequential procedures designed for binary outcomes, with no loss of statistical rigor. Notably, our empirical results show that competing policies can be separated more quickly when using fine-grained task progress than binary success metrics.
\end{abstract}
\IEEEpeerreviewmaketitle
\section{Introduction}
\label{the_introduction}
Recent advances in robot policy synthesis incorporate increasing complexity throughout the design process, training on large-scale datasets, using sophisticated, stochastic architectures, and requiring commensurate increases in training resources. These advances have led to significant improvements in solving dexterous and long-horizon tasks~\cite{barreiros2025careful}, inferring semantic information from context~\cite{intelligence2025pi}, and safely interacting with numerous other autonomous agents~\cite{kusano_comparison_2025}. However, this complexity makes design decisions like the choice of dataset, the training or fine-tuning procedure, and the network architecture analytically opaque. The value of a novel intervention cannot be determined from first principles but instead requires rigorous analysis~\cite{agarwal2021deep, kress2024robot} of its effect on empirical performance. 

Unfortunately, rigorous analysis of empirical performance is a significant challenge in the robotics setting. Hardware evaluation is the gold-standard measure, but is expensive and slow to collect. Evaluations in simulation reduce the time burden, but continue to suffer from sim-to-real gaps, making them an imperfect proxy~\cite{aljalbout_reality_2025}. Further, evaluation metrics can be quite coarse. Binary measurement of task success or failure, the \emph{de facto} standard for measuring the performance of robot manipulation policies (particularly on hardware), can obscure valuable information about the robot behavior~\cite{kress2024robot, beck_sparc_2018}. For instance, a policy that completes $90\%$ of the task is clearly better than a policy that is frozen the whole time, yet their success rates would be identically $0\%$. Finally, rigorous evaluation must account for inherent uncertainty, including in environment configuration and policy action selection (which is often stochastic \cite{chi_diffusion_2025}). This uncertainty means that observed empirical performance is a noisy estimate of, and \emph{not equivalent to}, the true expected performance of the policy. 

Since the value of design interventions is implicitly counterfactual (see \Cref{fig:anchor}), policy comparison is a particularly important form of evaluation to reliably determine the effects of design changes. For example: `does a change in policy architecture \cite{kim2024openvla, octo_2023} or action tokenization \cite{pertsch_fast_2025} improve downstream performance?' Or: `what set of expert data is the most valuable \cite{zha_guiding_2025} to improve policy performance?' Importantly, comparison must account for uncertainty in evaluating both the new design and the baseline. Making rigorous, statistically assured decisions for these problems is necessary to ensure reliable progress within the field. 

The costs of robot evaluation and the complexity of the robot policy design space motivate the development of a versatile framework for policy comparison. Ideally, such a framework would apply to very general performance measures, maintain statistical rigor, and ensure maximal sample efficiency via sequentialized evaluation. Current methods are deficient in at least one of these respects. Active learning \cite{anwar2025efficient} and asymptotic \cite{welch_generalization_1947, chernoff_sequential_1961} approaches are not statistically rigorous, particularly in small-data regimes. Batch and fully nonparametric methods \cite{Fisher_1922, barnard_significance_1947, welch_generalization_1947, waudby2024estimating} are not maximally sample efficient: the former due to the batch evaluation, and the latter due to not tailoring to the comparison setting. Near-optimal approaches \cite{snyder2025your, lai_nearly_1994, turner_exact_2023} are constrained to narrow metrics (e.g., binary success rate), and do not readily generalize. 

\begin{figure*}[th!]
    \centering
    \includegraphics[width=\linewidth]{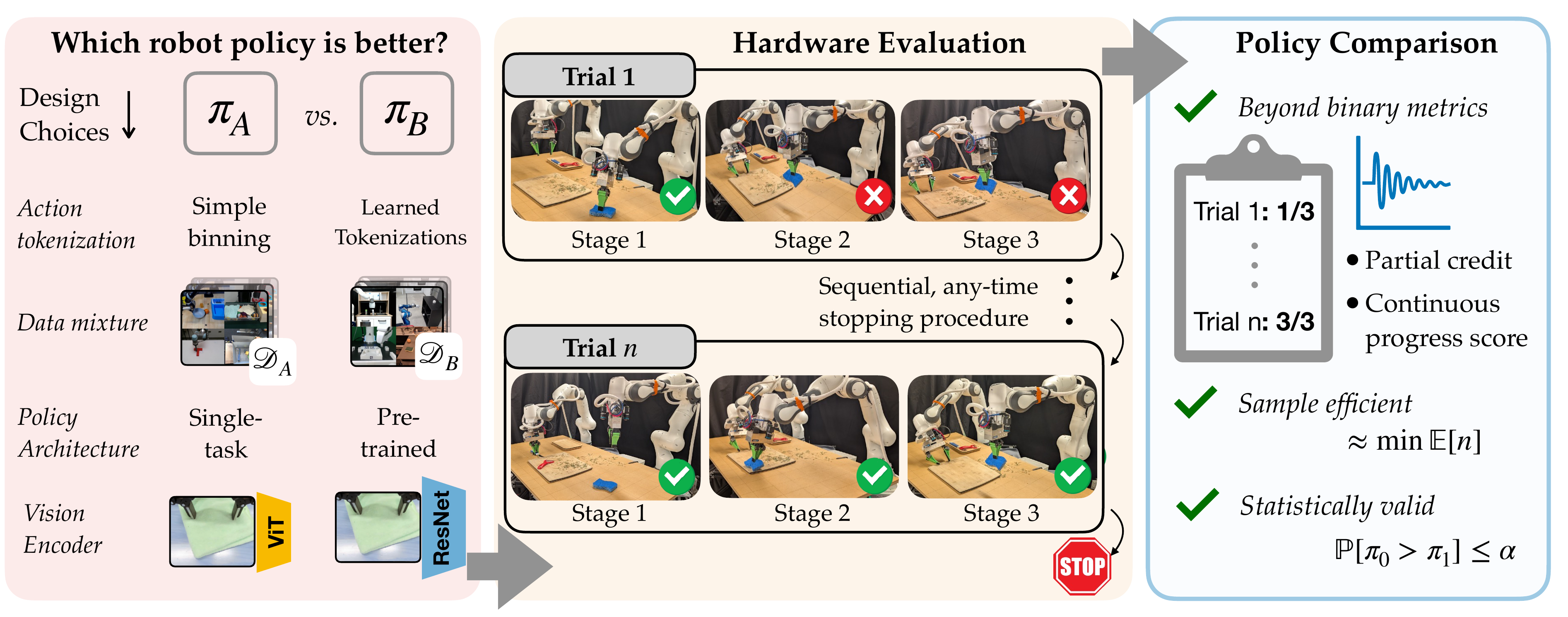}
    \caption{The evaluation context of the~\OurMethod~procedure. (Left) We consider the general problem of policy comparison, which arises out of counterfactual design decisions in the policy synthesis process. (Middle) Evaluation on hardware or in high-fidelity simulation is the gold standard to assess the effect of such changes, but is costly to collect. (Right)~\OurMethod~is a sequential evaluation procedure that is statistically rigorous, sample efficient, and generalizes to rich, diverse measures of robot performance.}
    \label{fig:anchor}
    \vspace{-1.5em}
\end{figure*}

This work presents \textbf{N}onparametric \textbf{S}equential \textbf{CO}mparison for \textbf{R}igorous \textbf{E}valuation (N-SCORE) to compare robot policy performance.~\OurMethod~is designed to address the shortcomings of prior approaches by explicitly accounting for all three of the preceding desiderata -- rigor, sample efficiency, and generality -- within the decision rule synthesis. We summarize our \textbf{contributions} as follows:  
\begin{enumerate}
    \item We introduce~\OurMethod, a novel method for sequential policy comparison with general progress metrics. 
    \item We prove that~\OurMethod~is statistically rigorous in the sense of controlling Type-1 Error, and empirically demonstrate its improved sample efficiency compared to state-of-the-art (SOTA) baselines.
    \item We comprehensively evaluate~\OurMethod~on large, high-impact evaluation datasets in robotics comprising over 4500 hardware evaluation rollouts and 2000 high-fidelity simulation rollouts \cite{barreiros2025careful, atreya2025roboarena}.
    The results demonstrate significant savings in evaluation effort of up to 70\% over batch methods and 50\% over methods using binary success measures to rigorously certify policy improvements. 
\end{enumerate}

\section{Related Work}
\label{the_related_work}
Efficient evaluation is increasingly appreciated as a critical aspect of the robot policy design process. The fundamental concern of these approaches is the strict limit on available hardware data. A core focus lies in constructing proxy signals (e.g., using simulations or evaluator preferences) to circumvent this constraint. By contrast, statistical testing approaches seek to employ the available hardware data more efficiently. These approaches are complementary; improvements in one domain augment those in the other. 
\subsection{Robot Policy Evaluation}
\label{evaluation_methods}
\noindent
\textbf{Hardware and Simulation Evaluation.} In robot manipulation, evaluation constraints often limit evaluation to 10-60 trials; the statistical validity of such comparisons is often not reported~\cite{snyder2025your,kress2024robot}. To address this constraint, recent works have introduced standardized benchmarks~\cite{heo2023furniturebench,luo2025fmb,yang2019replab,khargonkar2024scenereplica,collins2023ramp}, cloud-based evaluation platforms~\cite{zhou2023train,liu2021ocrtoc,bauer2022real,zhou2025autoeval}, and distributed evaluation infrastructure~\cite{atreya2025roboarena} to increase available data. Further, recent works have proposed active sampling of policy-task pairs to improve efficiency~\cite{anwar2025efficient}. However, none of these approaches confer statistical rigor on downstream comparisons. Similarly, though policy evaluation via simulation~\cite{todorov2012mujoco,makoviychuk2021isaac,pumacay2024colosseum,liu2023libero,li2024evaluating} or via world models ~\cite{veorobotics2025,tseng2025scalable,guo2025ctrl,1x_world,li2025worldeval,quevedo2025evaluating,zhu2024irasim,huang2025enerverse} show promise, the sim-to-real gap \cite{aljalbout_reality_2025} makes rigorous comparison difficult. Some recent works have proposed statistically rigorous methods to combine small-scale hardware testing with large-scale simulation evaluation~\cite{badithela2025reliable,luo2025leveraging}, but these do not address policy comparison. 

\noindent \textbf{Evaluation Metrics.} Recent studies emphasize the importance of detailed rubrics and task progress scores~\cite{kress2024robot,barreiros2025careful,atreya2025roboarena} over traditional binary success metrics. In parallel, finetuning procedures~\cite{intelligence2025pi} and policy ranking problems~\cite{chiang2024chatbot, atreya2025roboarena} have introduced indirect signals of policy performance, including reinforcement learning (RL) rewards and human preferences. However, none of the approaches give rigorous methods for decision-making. Whereas the state-of-the-art procedure for policy comparison under binary success metrics~\cite{snyder2025your} is rigorous and sample-efficient, current methods for more informative signals are not. This paper introduces a novel test procedure that scales rigorous evaluation to general metrics.

\subsection{The Parametric Policy Comparison Problem}
\label{parametric_policy_comparison_methods}
The \emph{parametric} setting for policy comparison arises when the evaluator designs the performance measure to have definite distributional structure (e.g., binary success or discrete partial credit). This structure specifies a tradeoff between the generality, sample efficiency, and correctness of the comparison. 

Classic tests \cite{Fisher_1922, barnard_significance_1947, boschloo_1970} were designed to optimize power for Bernoulli (binary) outcomes. Some subsequent results use tools from sequential analysis \cite{siegmund1985sequential} to improve the sample efficiency \cite{wald_sequential_1945, wald_optimum_1948, lai_nearly_1988}. In this regime, the STEP procedure~\cite{snyder2025your} is state-of-the-art. Other developments extend the generality of these tests to broader classes of parametric distributions~\cite{lai_nearly_1994, turner_exact_2023}. However, simultaneously extending both parametric generality and sample efficiency is difficult, requiring an exact solution of the Wald-Bellman partial differential equation~\cite{van_moerbeke_optimal_1974, lai_nearly_1988, caffarelli_geometric_2005}. For more complex parametric families, finding this solution quickly becomes intractable \cite{Chernoff_1986, fauss_minimax_2020}. 

This motivates work that extends generality by making approximations in the large-data regime. A canonical example is Welch's t-Test~\cite{welch_generalization_1947}, which is often used for batch comparison~\cite{barreiros2025careful, duan_benchmarking_2016, west_best_2021}. Further \emph{asymptotic} results~\cite{bickel_mathematical_2015} can extend generality~\cite{chernoff_sequential_1961, chernoff_sequential_1965_III, chernoff_sequential_1965_IV} \emph{and} improve sample efficiency~\cite{lai_optimal_1973, lai_boundary_1976, lai_first_1977, lai_boundary_1988}, at the cost of losing rigorous statistical assurances. This cost is significantly more pronounced in the low-data regime common to robotic evaluation~\cite{kress2024robot}. Thus, in our setting, conclusions derived from asymptotic approaches are difficult to assess, as their validity is not guaranteed in the low-data regime.

\subsection{Nonparametric Methods and Safe, Anytime-Valid Inference}
\label{nonparametric_methods}
Nonparametric methods seek to exchange some exploitable structure for generality in application. Within the batch regime, these are often termed `distribution free,' as in the case of conformal prediction and related techniques \cite{vovk_algorithmic_2022, shafer_tutorial_nodate, angelopoulos_gentle_2022}. However, they have two downsides: limited capacity to represent the underlying data distribution, and limited generalization to sequential settings. Kernel density estimation (KDE) addresses the first constraint by directly constructing a representation of the data-generating process \cite{parzen_estimation_1962}; see \cite{chen_tutorial_2017} for an overview. Critically, KDE is more efficient in low-dimensional settings \cite{jiang_uniform_2017} because it can quickly adapt to structure in the data. However, alone, it is not generally valid in finite samples \cite{wasserman_all_2006}. 

To address the second constraint, work in the line of safe, anytime-valid inference (SAVI) considers \emph{sequentialization} of estimation and decision problems, which act in an \emph{online fashion} \cite{ramdas2023game}. A core benefit is safety in the face of p-hacking \cite{john2012measuring}, or `data dredging'~\cite{grunwald_safe_2024, ramdas2023game}. Additionally, though not parametric, the framework does allow for methodological fine-tuning to the particular problem setting (as in \cite{cho_peeking_2024, lindon_anytime-valid_2025, turner_exact_2023}). Of these approaches, the WSR method~\cite{waudby2024estimating} is closest to our own, considering the problem of estimating the mean of a general class of random variables (see~\Cref{defn:progress_metric}). Importantly, ~\OurMethod~ utilizes ideas from KDE to tailor a SAVI framework \emph{specifically to the comparison problem}, achieving state-of-the-art generality and sample-efficiency. 

\section{Preliminaries}
\label{the_preliminaries}
We model a robot that must complete a task while subject to stochastic uncertainty in its environment as a Partially Observable Markov Decision Process (POMDP) \cite{kaelbling1998planning}. The uncertainty might arise in the initial environment configuration or in the robot's action selection. We assume the existence of a real-valued scalar performance measure~$R$ encoding the degree of success the robot demonstrates in completing the task. Conditioned on an arbitrary robot policy $\pi$ mapping state observations to actions, the stochastic environment uncertainty is compressed to uncertainty over performance outcomes (i.e., by running the policy and measuring its level of success).\footnote{This can be thought of as a pushforward measure (subject to necessary regularity assumptions).} Thus, all the potential complexities of the task, real-world system dynamics, and policy class are compressed to a (possibly complicated) distribution over the performance measure $\mathcal{D}_R$. Our goal is to make rigorous comparisons for as general a class of $\mathcal{D}_R$ as possible. With this in mind, we make a single restriction: that $R$ be a \emph{progress metric}, defined as a performance measure that is bounded w.p.~1.
\begin{defn}[Generalized Progress Metric]
\label{defn:progress_metric}
    A real-valued random variable $M(\omega)$ is termed a `\emph{generalized progress metric}' if it is bounded; that is, $\mathbb{P}[M \in [0, 1]] = 1$.\footnote{This immediately extends to real-valued performance measures that are bounded in an interval $[A, B]$ via normalization \cite{waudby2024estimating, jang_tighter_2023, orabona_tight_2024}.} 
\end{defn}
Unless stated otherwise, we assume $R$ is a progress metric as per \Cref{defn:progress_metric} but make no other structural assumptions. 

We conclude this section with a formalization of the comparison problem in the robotics context. In the most general form, we are tasked with quickly and accurately comparing the means of two distributions over a progress metric~$R$: 
\begin{equation*}
    \text{Test: } \mathbb{E}_{r \sim \mathcal{D}_R^{[0]}}[r] <  \mathbb{E}_{r \sim \mathcal{D}_R^{[1]}}[r]. 
\end{equation*}
This encompasses many practical problems, for example: comparing two policies with different architectures or trained on different data ($\pi_0$ and $\pi_1$), comparing distribution shift in task realizations for a single policy, or comparing the effects of different hyperparameters in training. Importantly, because the true means are unknown and not directly observable, a testing decision must be made on the basis of empirical performance in evaluation. Due to stochastic uncertainty in gathering this data, any decisions are inherently \emph{probabilistic} over the collection of evaluation data. Thus, we seek statistical assurances that bound the probability of an incorrect inference:
\begin{equation}
\label{example_conclusion_equation}
    \begin{split}
    & \text{If: } \hspace{5mm} \mathbb{E}_{\mathcal{D}_{R}^{[0]}}[R] \geq \mathbb{E}_{\mathcal{D}_{R}^{[1]}}[R] \\
    & \text{Then: }
    \mathbb{P}\bigg[\text{Wrongly decide }
    \mathbb{E}_{\mathcal{D}_{R}^{[0]}}[R] < \mathbb{E}_{\mathcal{D}_{R}^{[1]}}[R]\bigg] \leq \alpha^*.
    \end{split}
\end{equation}
The challenges to obtaining a result in the form of \Cref{example_conclusion_equation} are twofold. First, the decision must be statistically rigorous---it must safeguard against inadvertently reporting statistical noise as genuine improvement. This can require a substantial number of evaluations when the difference is small or the desired confidence $1 - \alpha^*$ is very high. However, the result also must be obtained as quickly as possible, so that the evaluator can efficiently allocate hardware or computational resources to begin investigating other problems. Indeed, allocating limited evaluation resources is a major bottleneck in improving policy performance \cite{snyder2025your}. The sequential nature of~\OurMethod~balances statistical confidence and speed of decision-making, adaptively making decisions when precisely enough evidence has accumulated and minimizing unnecessary evaluation trials. 
\section{Problem Formulation}
\label{the_problem_formulation}
\Cref{the_preliminaries} introduced the policy comparison problem, which seeks guarantees as in~\Cref{example_conclusion_equation}. The challenge in designing a test is the complexity of $\mathcal{D}_R^{[i]}$, as our procedure must be robust to \emph{any such progress metrics arising in practice}. 

\subsection{The Formal Hypotheses}
\label{constructing_the_hypotheses}
We formulate the problem in the context of frequentist statistical testing \cite{bickel_mathematical_2015}. To do so, we construct a general null (skeptical) hypothesis and an alternative (desired) hypothesis which reflect \emph{all possible} cases of \Cref{example_conclusion_equation}. 

In the general nonparametric case (see \Cref{nonparametric_methods}), we consider the set $\mathcal{M}_{[0, 1]}$ of all Lebesgue-measurable distributions on the real interval $[0, 1]$ (i.e., all progress metrics per \Cref{defn:progress_metric}). We partition $\mathcal{M}_{[0, 1]}$ into two parts:\footnote{For a parametric family $\Theta$ (e.g., Bernoulli distributions), the same formalism holds with a restriction to $\Theta^2 \subseteq \mathcal{M}_{[0, 1]}^2$}
\begin{equation*}
    \begin{split}
    S^{-} & := \{(\mathcal{D}_R^{[0]}, \mathcal{D}_R^{[1]}) \in \mathcal{M}_{[0, 1]}^2: \mathbb{E}_{\mathcal{D}_{R}^{[0]}}[R] \geq \mathbb{E}_{\mathcal{D}_{R}^{[1]}}[R]\}  \text{ and} \\
    S^{+} & := \{(\mathcal{D}_R^{[0]}, \mathcal{D}_R^{[1]}) \in \mathcal{M}_{[0, 1]}^2: \mathbb{E}_{\mathcal{D}_{R}^{[0]}}[R] < \mathbb{E}_{\mathcal{D}_{R}^{[1]}}[R]\}.
    \end{split}
\end{equation*}
Comparing the means amounts to formulating null and alternative hypotheses as $\mathcal{H}_0: (\mathcal{D}_R^{[0]}, \mathcal{D}_R^{[1]}) \in S^{-}$ and  $\mathcal{H}_1: (\mathcal{D}_R^{[0]}, \mathcal{D}_R^{[1]}) \in S^{+}$, respectively. 

These can be intuitively understood as ``all possible true states of the world in which the novel innovation is not better" ($\mathcal{H}_0$, the `skeptic'), and ``all possible true states of the world in which the novel innovation is better" ($\mathcal{H}_1$, the `desired result'). 
\Cref{example_conclusion_equation} amounts to being $1 - \alpha^*$ confident that the true state is \emph{not} in $S^-$. 

\subsection{Measuring the Quality of an Evaluation Algorithm}
\label{the_meta_metrics}
Having formulated the test hypotheses $\mathcal{H}_0$ and $\mathcal{H}_1$, we consider how to measure the quality of any potential evaluation scheme. These measures can be understood intuitively: the optimal evaluation protocol should \emph{make correct decisions} and \emph{make the decisions as quickly as possible}. 

There are precise mathematical analogues of these desiderata. In particular, there are two types of errors which can be made. First, in the case that $\mathcal{H}_0$ is true, the evaluation algorithm may incorrectly decide that $\mathcal{H}_1$ is true. This is termed a false positive, or a `Type-1 Error.' The rate at which such errors occur (under a particular decision-making protocol) is denoted $\alpha \in (0, 1)$. \Cref{example_conclusion_equation} precisely amounts to the statement ``the Type-1 Error rate of our evaluation method must be less than $\alpha^*$." In the case that $\mathcal{H}_1$ is true, the protocol may incorrectly decide that $\mathcal{H}_0$ is true. This is a false negative, or a `Type-2 Error.' The rate at which this occurs is denoted $\beta \in (0, 1)$. Semantically, a Type-1 Error corresponds to reporting a false discovery --- that the new policy is better \emph{when it actually is not}. A Type-2 Error corresponds to a missed discovery --- failing to discern that the new policy is better \emph{when it actually is}. These two error types combine to give a complete accounting of the algorithm's correctness. The last metric pertains to sample efficiency: how long the algorithm takes to make a decision, denoted $\mathbb{E}[N]$.\footnote{For technical reasons, sample efficiency must be posed with respect to a measure over $S^+$. This amounts in practice to prior beliefs over features like the size of the gap in performance; the reader may safely assume, e.g., an `uninformative' uniform measure.} As presented, an optimal evaluation algorithm --- one that is fast and correct --- minimizes all three of the measures $\{\alpha, \beta, \mathbb{E}[N]\}$. 

\subsection{Efficient Tests as a Multi-Objective Optimization}
Unfortunately, there are fundamental tradeoffs which prevent simultaneous minimization of all of these metrics. Therefore, we adopt the Neyman-Pearson testing approach \cite{neyman_1933}, which normatively chooses the Type-1 Error rate as the quantity to be rigorously controlled. This amounts to requiring that $\alpha \leq \alpha^*$ be treated as a hard constraint. Having made this selection, the problem of designing an evaluation procedure reduces to finding (near-)optimal solutions to the following multi-objective optimization problem: 
\begin{equation}
    \label{multi_objective_opt}
    \begin{split}
        \gamma^* = \arg \inf_{\gamma \in \Gamma} & \hspace{2mm} \bigg[ \mathbb{E}[N](\gamma) + \lambda \cdot \beta(\gamma) \bigg] \\
        \text{s.t.} & \hspace{2mm} \alpha(\gamma) \leq \alpha^*.
    \end{split}
\end{equation}
Here, $\Gamma$ denotes a space of decision-making rules (evaluation procedures) and $\lambda \in [0, \infty)$ is a nonnegative parameter trading off the expected time to decision and false negative rate. 

Efficiently solving \Cref{multi_objective_opt} results in an evaluation procedure that is guaranteed to maintain statistical rigor and quickly and effectively detects changes in performance. This provides the roboticist with the capacity to quickly iterate on new innovations, while ensuring justified and well-calibrated confidence in reported performance improvements. 
\section{Methodology}
\label{the_methodology}
We motivate the~\OurMethod~procedure via the construction of a process to distinguish $S^-$ from $S^+$. Intuitively, we will construct a scalar-valued dynamical system that behaves fundamentally differently when $\mathcal{H}_0$ is true than when $\mathcal{H}_1$ is true, based on the empirical evidence observed during evaluation. This difference is precisely in the sense of being stable in the former case and unstable in the latter. The testing problem then reduces to assessing the stability properties of the dynamical system; this assessment can be rigorously quantified within the SAVI framework described in \Cref{nonparametric_methods}. 

\subsection{Constructing an `Evidence Integrator'}
\label{the_general_recipe}
A natural correlate to $\mathcal{H}_1$ is the difference in the evaluation reward; if we are on the $n^{th}$ evaluation trial, we consider:
\begin{equation*}
    \Delta \text{ evidence}_n \approx \bigg[\xi \cdot (r_{1, n} - r_{0, n})\bigg], 
\end{equation*}
where $\xi > 0$ is a positive scaling factor and $r_{0, n}$, $r_{1, n}$ are the observed progress values. Intuitively, the evidence for $\mathcal{H}_1$ is positive when $r_{1, n} > r_{0, n}$ and negative otherwise. Note that by assumed independence of the evaluation trials, this relationship does not depend on $n$. What remains is to design the appropriate rate of integration of this evidence. Results in the SAVI literature and across a variety of statistical estimation contexts suggest that the optimal aggregation rate~\cite{ramdas2023game} is \emph{multiplicative}, resulting in a measure of aggregate evidence $X_n$ (setting $X_0 = 1$ w.l.o.g.) that evolves according to:
\begin{equation}
\label{the_evidence_integrator}
    X_{n+1} = (1 + \xi \cdot (r_{1, n} - r_{0, n}))X_n. 
\end{equation}
Therefore, when the evidence is positive for $\mathcal{H}_1$, the growth rate is greater than one, and the system is locally unstable. Conversely, when it is negative (i.e., in favor of $\mathcal{H}_0$), the growth rate is less than one and the system is stable. We can additionally optimize $\xi_n = g(\mathcal{F}_{n-1})$ online based on the evidence accumulated so far (the `natural filtration' $\mathcal{F}_{n-1}$), as long as the choice of $\xi_n$ is independent of $(r_{0, n}, r_{1, n})$.

The last step is designating an evidence threshold, amounting to the idea that ``$X_n$ has become sufficiently unstable as to provide sufficient evidence that the true state of the world is not in $\mathcal{H}_0$." We designate the threshold to be $1/\alpha^*$, for a desired Type-1 Error rate $\alpha^*$ of the test procedure. 
\begin{equation}
\label{stopping_rule}
    \begin{split}
    & \text{If: } \hspace{5mm}X_n \geq \frac{1}{\alpha^*} \\
    & \text{Then: }
    \bigg[\text{Stop at step }n, \text{ and conclude $\mathcal{H}_1$ is true.} \bigg]
    \end{split}
\end{equation}
This procedure is operationalized in \Cref{the_nsm_algorithm}. As shown in \Cref{the_nsm_analysis}, this evaluation protocol efficiently balances time-to-decision and correctness, while rigorously controlling the Type-1 Error rate at tunable, pre-specified level~$\alpha^*$. These properties enable rigorous confidence in comparison problems (yielding statements of the form of \Cref{example_conclusion_equation}) while minimizing the evaluation burden necessary to reliably form them. 

\begin{algorithm}[t]
\caption{~\OurMethod~Evaluation Protocol}\label{the_nsm_algorithm}
\begin{algorithmic}
\STATE \textbf{Input: } 
\STATE Type-1 error limit $\alpha^* \in (0, 1)$, evaluation limit $N_{\text{max}}>0$.
\STATE \textbf{Initialize: } 
\STATE $X_0 = 1$; $\bar{X} = 1$; $\xi_0 = 0$; $\mathcal{F}_0 = \{\emptyset\}$; $n=1$.
\WHILE {$\bar{X} < 1/\alpha^*$ \AND $n \leq N_{max}$}
\STATE Observe evaluation progress scores: $r_{0, n}$, $r_{1, n}$
\STATE Compute increment: $a_{n-1} = 1 + \xi_{n-1}(r_{1, n} - r_{0,n})$
\STATE Update martingale: $X_n \gets a_{n-1}\cdot X_{n-1}$
\STATE Update test statistic: $\bar{X} \gets \max \{\bar{X}, X_n\}$
\STATE Update filtration: $\mathcal{F}_n \gets \mathcal{F}_{n-1} \cup (r_{0, n}$, $r_{1, n})$
\STATE Update $\xi_n \gets \text{proj}_{[0, 1]} \hspace{1mm} g(\mathcal{F}_n)$
\STATE $n \gets n + 1$
\ENDWHILE
\IF {$\bar{X} < 1/\alpha^*$}
\RETURN \text{Fail to Reject Null}
\ELSE 
\RETURN \text{Reject Null}
\ENDIF
\end{algorithmic}
\end{algorithm}

\subsection{Theoretical Properties of \Cref{the_nsm_algorithm}}
\label{the_nsm_analysis}
We briefly present several key theoretical results of the evaluation protocol effectuated in \Cref{the_nsm_algorithm}. The import of these results, along with proof sketches, are included \emph{in situ}. Detailed proofs are deferred to the Supplement. 

We begin with a critical lemma pertaining to a property of the evidence aggregation process in \Cref{the_evidence_integrator}. This property amounts to a statement that the stochastic process $X_n$ is stable in expectation \emph{for all} elements $h \in \mathcal{H}_0$. 
\begin{lem}[Null Stability (NSM) Property]
\label{nsm_lemma}
    Consider the stochastic process $\{X_n\}$ defined in \Cref{the_evidence_integrator}, setting (w.l.o.g.) $X_0 = 1$. Then the expectation of $\{X_n\}$ is contracting in time with respect to the current value, for all $h \in S^{-}$ for any $\xi_n \in [0, 1]$. That is: 
    \begin{equation}
        \label{nsm_property_equation}
        \sup_{h \in S^-} \hspace{1mm} \sup_{\xi_n \in [0, 1]} \mathbb{E}_{r_0, r_1 \sim h} \bigg[\frac{X_{n+1}}{X_n} \bigg\rvert \mathcal{F}_{n}\bigg] \leq 1.
    \end{equation}
\end{lem}
\Cref{nsm_lemma} is necessary for ensuring stability of the process in \Cref{the_evidence_integrator} because it rules out (with high probability) that evidence increments generated by any element of $\mathcal{H}_0$ will cause the process to grow by very much. This is a necessary step to ensure Type-1 Error control and enforce the hard constraint in \Cref{multi_objective_opt}, which is formalized in \Cref{nsm_theorem}. 
\begin{thm}[Type-1 Error Control of \Cref{the_nsm_algorithm}]
\label{nsm_theorem}
    Consider the evaluation procedure in \Cref{the_nsm_algorithm}, utilizing the process defined in \Cref{the_evidence_integrator}. Then 
    \begin{equation}
    \label{theorem_equation}
        \mathbb{P}\bigg[\mathcal{H}_0 \text{ is true} \bigg \rvert \text{ decide }\textbf{Reject Null} \bigg] \leq \alpha^*
    \end{equation}
\end{thm}
\begin{proof}
    The proof uses \Cref{nsm_lemma} in concordance with Ville's Inequality \cite{Ville_1939} to bound the probability of observing large values of $X_n$. After verifying several ancillary conditions and matching constants, the bound can be computed directly. More detail can be found in the Supplement.  
\end{proof}
Enforcing Type-1 Error control ensures high confidence (at level $1 - \alpha^*$) that reported significant innovations and improvements to the state-of-the-art are strictly separated from the inherent noise in robotic evaluation. 
\begin{rem}[Efficient Optimization of $\xi_n$]
    \label{time_varying_xi}
    Consider the stochastic process family described in \Cref{the_evidence_integrator}, and let $\mathcal{F}_n$ be the natural filtration $\{(r_{0,i}, \hspace{1mm}r_{1,i})\}_{i=1}^{n-1}$ at step $n$. There exists an efficient algorithm to maximize (online) over $\{\xi_n\}_{n=1}^N$ the expected growth rate of $\{X_n\}_{n=1}^N$. 
\end{rem}
Intuitively, $\xi_n$ modulates the degree of confidence in marginal changes to $X_n$. Large $\xi_n$ grow the process more quickly when data is favorable, but are penalized more harshly when it is not. Identifying an effective strategy to modulate $\xi_n$ online is central to improving sample efficiency across a broad array of evaluation problems. To do this,~\OurMethod~utilizes intuition from kernel density estimation and the structure of the discrete partial credit setting to optimize $\xi_n$ via constructing explicit nonparametric representations of $\mathcal{D}_R^{[i]}$. This is represented as a family~\OurMethod$_{k}$, where $k \in \mathbb{N}$ is a real-valued parameter akin to the kernel bandwidth. For brevity, details about this optimization are deferred to the Supplement. However, we emphasize that this optimization has strong practical benefits, yielding a state-of-the-art near-optimal solution to \Cref{multi_objective_opt} for general progress metrics. 
\section{Experimental Results}
\label{the_experiments_section}
We evaluate~\OurMethod~across a broad suite of evaluation settings, including some of the largest available datasets for real-world robot comparison with task progress metrics, such as RoboArena~\cite{atreya2025roboarena} and the LBM~1.0 study~\cite{barreiros2025careful}. To organize the results, we tie them to three core research questions (RQs): 
\begin{enumerate}
    \item (\textbf{Sequential Evaluation}) Are there sample efficiency benefits of sequential evaluation?
    \item (\textbf{Informative Metrics}) Are there sample efficiency benefits from using more informative evaluation metrics than coarse binary success?
    \item (\textbf{Technical Novelty}) What are the benefits of ~\OurMethod~with respect to statistically valid sequential policy comparison approaches? 
\end{enumerate}

\subsection{Baseline Methods}
\label{the_baseline_methods_section}
We introduce three relevant baselines, proceeding in order of increasing generality. The recently proposed STEP procedure \cite{snyder2025your} demonstrated SOTA performance in the setting of binary success metrics. However, STEP is tailored to this narrow, but important setting and does not generalize to more informative metrics. Another recent work proposed a safe, anytime-valid inference approach to parametric comparison problems, motivated by settings with discrete partial credit \cite{turner_exact_2023}. This method, termed $\theta$-SAVI to denote its parametric nature, is more general than STEP, as it retains validity for any parametric comparison setting (defined in \Cref{parametric_policy_comparison_methods}). As with STEP, $\theta$-SAVI cannot extend to nonparametric performance metrics such as continuous-valued progress scores. The most general evaluation procedure is the test dual to the WSR estimation method \cite{waudby2024estimating}, which is also most similar among the baselines to our approach. Thus, the key comparison between WSR and~\OurMethod~will center on sample efficiency.

\subsection{RQ1 and RQ2: Sequential Evaluation with Informative Metrics}
\label{the_first_research_question}
Our experiments support the observations in \citet{snyder2025your} that sequential test procedures save significant time and resources for robotics evaluations under binary success metrics. We further observe the benefit of  sequential procedures to extend to more fine-grained evaluation metrics. 

\begin{table}[]
    \centering
    \resizebox{0.85\columnwidth}{!}{%
    \begin{tabular}{c|cccc}
    \hfill & \multicolumn{2}{c}{\textbf{Bernoulli Data}} & \multicolumn{2}{c}{\textbf{Nonparametric Data}} \\
    \toprule
        & TTD & Power & TTD & Power \\
        \midrule
        STEP & 95.1 & 0.953 & -- & -- \\
        SAVI & 117.6 & 0.962 & -- & -- \\
        ~\OurMethod$_{2}$ & 117.9 & 0.965 & -- & -- \\
        \midrule 
        ~\OurMethod$_{\infty}$ & 122.3 & 0.958 & 206.8 & 0.889 \\
        WSR & 224.8 & 0.592 & 247.3 & 0.840 \\
        \bottomrule
    \end{tabular}
    }
    \caption{\textbf{Average time-to-decision (TTD) and empirical power for each method on simulated data.} For Bernoulli data (left), TTD is averaged over all 35 alternative hypotheses, while Power is averaged only the 9 hardest alternative hypotheses corresponding to a gap of $0.1$ (all other 26 alternatives have empirical power 1). For each alternative, metrics are averaged over 250 independent redraws; each redraw has $N=1000$. We observe similar statistical power for all methods except WSR, which lags substantially due to high variance in the Bernoulli regime. For nonparametric data (right), TTD and power are averaged over 3000 independent redraws (there is only one alternative). Only~\OurMethod$_\infty$ and WSR are valid in the nonparametric regime; the remaining methods cannot be adapted to this case. We observe that~\OurMethod$_\infty$ outperforms WSR by approximately 15\% on average in terms of time-to-decision, while improving the empirical power by approximately 5 percentage points.} 
    \label{the_simulated_data_table}
\end{table}
\subsubsection{Results on Artificial Bernoulli Sequences}
The left-hand side of \Cref{the_simulated_data_table} shows the average time-to-decision for each baseline and~\OurMethod~on artificially generated Bernoulli data comprising 35 different distributions, where each baseline is tested on each distribution 250 times. In each instance, the batch evaluation size was set to $N = 1000$ samples per test. Further visualizations of these tests are deferred to the Supplement. The average time-to-decision is shown to be significantly less than 1000 samples, illustrating the practical benefit of early stopping when the evaluation problem is sufficiently easy as to not require the full batch allocation.

\begin{figure}
    \centering
    \includegraphics[width=\linewidth]{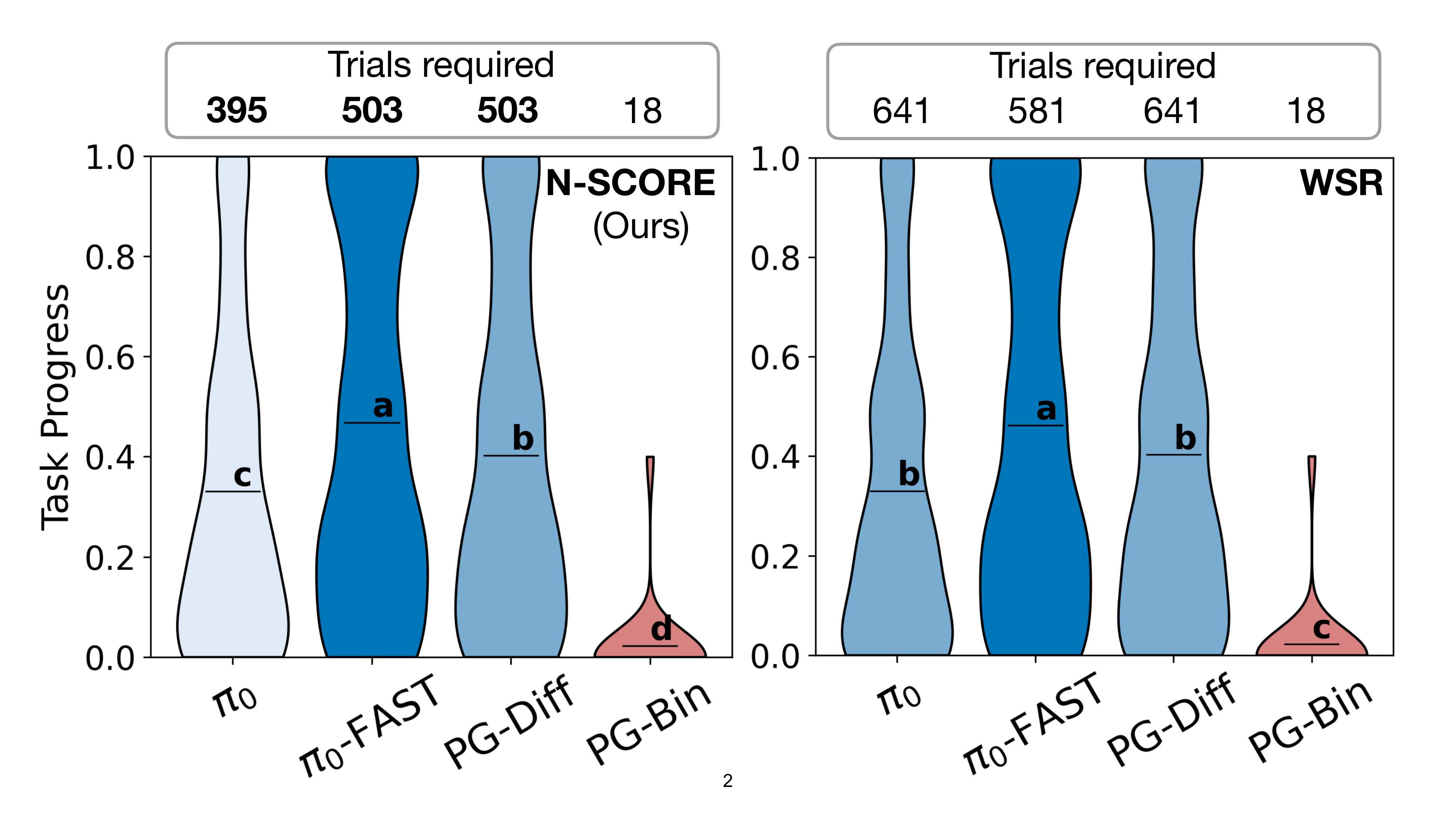}
    \caption{\textbf{Policy performance comparisons on crowd-sourced real-world evaluations on the DROID~\cite{khazatsky2024droid} setup from RoboArena~\cite{atreya2025roboarena}}. The violin plots represent empirical distributions of observed results. Policies with different letters are statistically distinguishable by the method. Policies are compared at a global error bound of $\alpha=0.05$ with a Bonferroni correction.}
    \label{fig:roboarena}
    \vspace{-2em}
\end{figure}

\subsubsection{Results on LBM 1.0 Binary Evaluation \cite{barreiros2025careful}}
We now consider real-world binary evaluation data for robot policies from the LBM~1.0 study results~\cite{barreiros2025careful}. The binary success metric results are on the right side of \Cref{table_of_real_data}. In this example, a pretrained and then finetuned policy is compared to a single-task policy trained from scratch to evaluate performance; it is desired to test whether the pretrained policy outperforms the single-task policy on challenging, long-horizon tasks. The tasks and rollout data come from the experimental setup of LBM 1.0~\cite{barreiros2025careful}. The top half considers robot performance comparison in high-fidelity simulator data. For each task and method, the time-to-decision (TTD) is recorded in terms of the number of evaluations per policy before a comparison outcome at level $\alpha = 0.05$ was reached. The total number of nominal batch evaluations is $2000$. We observe savings of $25\%-35$\% from the use of sequential evaluation procedures, resulting in an empirical reduction in the requisite number of simulations of approximately 500-700. In the bottom half, we consider the same comparison on hardware. In this setting, the nominal batch evaluation burden is $500$. We again observe significant savings, on the order of $16\%-25$\%, corresponding to a savings of up to $125$ total evaluations on this task suite. 

\begin{table}[t!]
    \centering
    \resizebox{\columnwidth}{!}{%
    \begin{tabular}{lcc|cccc}
    \toprule
    \hfill Method $\rightarrow$& \multicolumn{2}{c|}{\textbf{Progress Metrics}} & \multicolumn{4}{c}{\textbf{Binary Success/Failure Metrics}} \\
    Task $\downarrow$ &~\OurMethod$_{\infty}$ & WSR & STEP &~\OurMethod$_{2}$ & $\theta$-SAVI & WSR \\
    \midrule
    DumpVegetables  & 38 & 36 & 154 & -- & -- & -- \\
    PutContainer    & 33 & 36 & 169 & -- & -- & -- \\
    PutFruit        & 10 & 10 & 40  & 46 & 45 & 52 \\
    SeparateFruit    & 18 & 20 & 77  & 98 & 98 & 103\\
    TurnUpsideDown   & -- & -- & --  & -- & -- & -- \\
    \midrule
    Total (Simulation)  & 598 & 604 & 1280& 1488 & 1486 & 1510 \\
    \midrule
    \midrule
    BikeRotor        & -- & -- & -- & -- & -- & -- \\
    CutApple         & 29 & 29 & -- & -- & -- & -- \\
    CleanLitter      & 36 & 23 & -- & -- & -- & -- \\
    ClearCounter     & 16 & 15 & 23 & 25 & 30 & 46 \\
    SetUpBreakfast   & 12 & 13 & 15 & 19 & 19 & 17 \\
    \midrule
    Total (Hardware) & 286 & 260 & 376 & 388 & 398 & 426 \\
    \bottomrule
    \end{tabular}
    }
    \caption{\textbf{Time-to-decision for all \textbf{simulation (top)} and \textbf{hardware (bottom)} policy comparisons for evaluation context in \cite{barreiros2025careful}}. If a decision is not reached, the entry is left blank; for the purpose of computing evaluation savings, any blank entry is counted at $N$ trials. All simulation tasks utilize $N = 200$; all hardware tasks utilize $N = 50$. The total number of trials is the column sum multiplied by two (because there are two policies being evaluated, and each must be run). Thus, a column in the top half could require up to 2000 simulated trajectories; one in the bottom half could require up to 500 hardware evaluations. Several important observations are: (1) sequential evaluation saves significant evaluation effort over batch methods (right); (2) more informative partial credit metrics can provide \emph{even greater} savings (left). In fact, for these evaluations, the savings are up to 50\% on hardware and 70\% in simulation.}
    \vspace{-2em}
    \label{table_of_real_data}
\end{table}
RQ2 considers the efficiency of evaluation in settings which go beyond binary metrics. Following the structure of RQ1, we again consider significant evidence from both simulated and real-world evaluation data. Here, however, the core purpose is to illustrate the generality of nonparametric approaches, which open up opportunities for rigorous comparison of richer metrics, including RL rewards and continuous progress metrics.

\subsubsection{Results on Simulated Nonparametric Data}
\label{results_on_simulated_nonparametric_data}
We again start with results on simulated data, now generated from nonparametric densities. Random polynomials of order up to 10 are generated, and then rectified into an appropriate density (they are shifted and scaled to be everywhere nonnegative on $[0, 1]$ and to integrate to 1). This process is analytically opaque, making it difficult to express the resulting family of distributions in any parametric form. Thus, only~\OurMethod~and WSR are applicable. We run $3000$ comparison sequences of pairs of these distributions, limiting to cases where the gap in mean performance is at least $0.01$. $N = 1000$ for each sequence. The right-hand side of \Cref{the_simulated_data_table} shows summary statistics for~\OurMethod~and WSR, respectively. Again, the average time-to-decision is very low in comparison to the batch allocation, suggesting that distribution complexity does not impede efficient comparison, and does not attenuate the sample efficiency benefits of this evaluation paradigm. 

\subsubsection{Results on LBM 1.0 Partial Credit Evaluation~\cite{barreiros2025careful}}
We repeat the analysis of the task and rollout data obtained from \citet{barreiros2025careful}, as described in the preceding section, but using the discrete partial credit rubric instead of binary success metrics. Before each task was evaluated, partial credit measures (e.g., completion rate of $K$ subtasks) were designed to provide fine-grained information on the policy performance. In every task, this varied between six and eight partial credit outcomes; each subtask was weighted equally in the scoring. As before, the top half corresponds to policy evaluation in high-fidelity simulation, with a total evaluation budget of $2000$ rollouts. As shown in the left-hand side of \Cref{table_of_real_data}, partial credit metrics yield savings of approximately $70$\%, corresponding to a nearly $1400$-sample reduction in the number of evaluations for each sequential procedure. Further, this corresponds to an equivalent improvement over sequential binary methods, such as STEP, of over $50$\%. On hardware trials, the results are similar in trend. We observe approximately $45$\% reduction in the number of required evaluations with respect to the batch procedure (which arbitrarily chooses 50 runs per task per policy). This corresponds to a $24\%-30\%$ reduction in evaluation burden with respect to SOTA binary procedures like STEP \cite{snyder2025your}. The preceding evaluation gives strong empirical evidence for answering RQ1 and RQ2 in the affirmative: sequentialized evaluation reliably improves sample efficiency on artificial and real-world data across a wide range of metrics and uncertainty distributions.

\subsection{RQ3: Improving Sample Complexity Over Baselines}
\label{the_technical_novelty_results}
Finally, we discuss instances of performance differences between~\OurMethod~and related baselines. 

\subsubsection{Simulated Nonparametric Data}
From the right-hand side of \Cref{the_simulated_data_table}, there is a moderate gap in average time-to-decision, corresponding to average savings of around 15\% in evaluation burden. This can be understood as the aggregate effect of the efficient mechanism to optimize $\xi_n$ (see \Cref{the_nsm_algorithm}), which does not have a clear analogue in the WSR approach. Furthermore, neither STEP nor the $\theta$-SAVI approaches can be applied in this setting. 

Due to the generality of the nonparametric testing paradigm, the times-to-decision can be statistically compared using the~\OurMethod~framework. To do this, the 3000 runs are randomly partitioned into two sets, each of size 1500, corresponding to the time-to-decision of one of the two methods (\OurMethod~or WSR) on that data. These can be sequentially compared using the metrics $r_{i, n} = \frac{\text{ttd}_{i,n}}{N}$. Note that to run this evaluation in a parametric model would require dimension $N+1 = 1001$. The~\OurMethod~procedure returns a decision at $\alpha=0.05$ in approximately $525$ evaluations, corresponding to the hypothesis: ``the~\OurMethod~procedure has a lower time-to-decision than WSR on this class of nonparametric densities w.p. $\geq 0.95$."

\subsubsection{Simulated Bernoulli Data}
We consider the left-hand side of \Cref{the_simulated_data_table} to specifically consider the differences between sequential evaluation methods. Here, all baselines are applicable. We observe that STEP is optimal, as expected in this setting. However, among the remaining methods, $\theta$-SAVI and~\OurMethod~perform nearly identically, suggesting that our approach is able to efficiently approximate the available parametric structure. WSR, conversely, suffers substantially in the time-to-decision as compared to these approaches. 

\subsubsection{LBM 1.0 Success and Partial Credit}
Now, we revisit the results in \Cref{table_of_real_data} to consider the implications for each evaluation procedure. Here, we observe that STEP is again optimal for binary success measures, consistent with the previous section. For this data, $\theta$-SAVI,~\OurMethod~, and WSR are equally effective in both binary and discrete partial credit evaluation on this dataset.

\subsubsection{Results on RoboArena}
We continue the analysis of real-world evaluation data by illustrating the efficacy of~\OurMethod~, in addition to its applicability to multi-policy comparisons using continuous progress evaluation scores, in simultaneously comparing four policies from the open-source RoboArena~\cite{atreya2025roboarena} benchmark. Further details of the dataset are deferred to the Supplement. Because the rewards are continuous, only~\OurMethod~and WSR can be applied.~\Cref{fig:roboarena} illustrates the time-to-decision (TTD) in terms of the number of trials required for each policy by~\OurMethod~and WSR.~\OurMethod~is able to distinguish the performance of all policies. In contrast, while WSR is able to correctly distinguish the best policy as {\(\pi_0\)-FAST}, it is unable to separate {\(\pi_0\)} and {PG-Diff} even after exhausting all available 641 trials. Notably,~\OurMethod~requires over 200 fewer trials of {\(\pi_0\)}, and results in a total savings of at least \textbf{450 trials} (1419 vs. 1881). The efficacy of our method can be attributed to efficient optimization of \(\xi_n\) (see~\Cref{time_varying_xi}) with the available data, and due to WSR not being optimized for policy comparison. 

\begin{table}
    \centering
    \resizebox{\linewidth}{!}{%
    \begin{tabular}{lcc|cc}
    \toprule
    \hfill Comparison $\rightarrow$
        & \multicolumn{2}{c|}{\textbf{PPO vs. DDPG}}
        & \multicolumn{2}{c}{\textbf{SAC vs. TD3}} \\
    Task $\downarrow$ \hfill Method $\rightarrow$
        & WSR & ~\OurMethod$_{\infty}$
        & WSR & ~\OurMethod$_{\infty}$ \\
    \midrule
    Ant-v4                  & 27  & 21  & 14  & 13 \\
    HalfCheetah-v4          & 8   & 8   & 18  & 15 \\
    Hopper-v4               & 13  & 12  & 30  & 20 \\
    InvertedPendulum-v4     & 677 & \textbf{267} & --  & -- \\
    Humanoid-v4             & 42  & 40  & 96  & 89 \\
    Walker2d-v4             & 23  & 22  & 24  & 22 \\
    Pusher-v4               & 89  & 77  & 249 & 220 \\
    \midrule
    Total (14000 nominal)   & 1758 & \textbf{894} & 2862 & \textbf{2758} \\
    \bottomrule
    \end{tabular}
    }
    \caption{\textbf{Time-to-decision for selected \textbf{reinforcement learning} policy comparisons on Mujoco benchmarks.} If a decision is not reached, the entry is left blank; for the purpose of computing evaluation savings, any blank entry is counted at $N$ trials. All simulated tasks utilize $N = 1000$. We observe similar behavior between~\OurMethod~and WSR on easier instances (with lower times-to-decision); however, in harder instances significant improvements can be observed. In aggregate, the hard instances dominate sample complexity, resulting in substantial savings in evaluation burden.}
    \vspace{-2.0em}
    \label{table_of_rl_data}
\end{table}

\subsubsection{Real-valued Continuous Metrics}
\label{the_real_world_nonparametric_results}
The final example we consider is determining the best reinforcement learning (RL) policy based on continuous-valued episodic rewards, whose underlying distribution is difficult to model.~\Cref{table_of_rl_data} lists the time-to-decision in distinguishing popular RL algorithms (PPO, TD3, DDPG, and SAC) on Mujoco~\cite{todorov2012mujoco} benchmarks. Due to the continuous nature of the reward metric, only~\OurMethod~and WSR are applicable. We observe significant improvement over the WSR baseline when policies perform similarly but have high variance such as in the case of InvertedPendulum-v4. In this instance,~\OurMethod~saves over 400 trials when comparing PPO with DDPG. For the same benchmark, SAC and TD3 are highly effective, returning the maximum possible reward in each rollout, thereby leading to no statistical separation by either method. We provide further experimental details including mean episodic return and violin plots in the Supplement. 

As demonstrated by these extensive empirical validations, the key impact of our novel approach lies in effectively matching the sample efficiency of $\theta$-SAVI in parametric contexts (e.g., \Cref{the_simulated_data_table} and \Cref{table_of_real_data}), while maintaining the generality of, and improving sample efficiency over, the WSR procedure in nonparametric contexts (e.g., \Cref{table_of_rl_data} and \Cref{fig:roboarena}).
\section{Limitations and Future Work} 
\label{the_limitations_and_future_work}
There are several current limitations of~\OurMethod, which suggest the possibility for valuable future investigation. First, unlike STEP, any procedure using tools from safe, anytime-valid inference (SAVI) tends to achieve tighter Type-1 error control than specified, leaving some `risk budget' unused. This both explains the gap to STEP in the regime of binary evaluation metrics and the capacity for robust generalization to complex and nonparametric measures. 
Developing a finite-$N$ rectification to use the full available risk budget would be exceedingly valuable for a host of problems for which SAVI is currently applied.
Similarly, the current method for optimizing $\xi_n$ has connections to ideas in kernel density estimation, but at present the full insight of developments in the latter have not been applied to our approach. Using these tools and domain knowledge promises more efficiency and potential generalization to unbounded performance measures. 

We note that, while SAVI methods naturally hinder some avenues towards inadvertent data dredging, they rely crucially on i.i.d. evaluation data. Moreover, rigorous guarantees are only meaningful if the evaluation paradigm is similarly rigorous and reproducible. As such, these methods are inherently dependent on careful and measured evaluation procedures. 
\section{Conclusion} 
\label{the_conclusion}
We have introduced and validated a novel procedure for statistically rigorous sequential evaluation of robot policies, generalized to metrics that go beyond binary success and failure rates. In so doing, we have highlighted the practical benefits of sequential methods and informative metrics to reduce evaluation burden, and situated our approach as a novel synthesis of two state-of-the-art sequential evaluation procedures. Each of these results is validated by substantial empirical evidence spanning simulated and real-world evaluation data. The promise of such results is to both codify \emph{and accelerate} progress within the field, by ensuring reliable performance improvements and minimizing the requisite evaluation burden necessary to confirm them.

\section*{Acknowledgments}
D. Snyder acknowledges support from the Toyota Research Institute (TRI) and the Penn AI Fellowship. A. Badithela is supported by the Presidential Postdoctoral Fellowship. Additionally, the authors were partially supported by the NSF Career Award \#2044149, the NSF SLES Award \#2331880, and the Sloan Fellowship. TRI provided funds to assist the authors with their research; this article solely reflects the opinions and conclusions of its authors, and not TRI nor any other Toyota entity.

\bibliographystyle{plainnat}
\bibliography{main}

\newpage{}
\appendices

\renewcommand{\thefigure}{A.\arabic{figure}} 
\renewcommand{\thetable}{A.\Roman{table}}    
\setcounter{figure}{0}                       
\setcounter{table}{0}                        

\section{Analytical Results}
\label{analytical_supplement_results}
In this section, we present the proofs for key theoretical results in the main paper, in the order that they are introduced in \Cref{the_methodology}.
\subsection{Proof of \Cref{nsm_lemma}}
\label{proof_of_nsm_lemma}
Consider the stochastic process increment in \Cref{nsm_property_equation}, interpreted as the `approximate marginal evidence increment' represented in \Cref{the_evidence_integrator}. We need to show that:
\begin{equation}
    \sup_{h \in S^-} \sup_{\xi \in [0, 1]} \mathbb{E}_{r_0, r_1 \sim h} (1 + \xi(r_1 - r_0)) \leq 1.
\end{equation}
The property arises directly out of the boundedness assumption of general progress metrics and the linear separability of $r_0$ and $r_1$ in the increment computation. From boundedness on $r_0$ and $r_1$, we know that the increment is bounded w.p. 1 in $[1-\xi, 1 + \xi] \subseteq [0, 2]$. Therefore, the increment has well-defined moments. 

The exact value can be computed as: 
\begin{equation*}
    \begin{split}
    \mathbb{E}[1 + \xi(r_1 - r_0)] & = \int 1 + \xi(r_1 - r_0) d\mu_{0, 1}\\
    & = \int_0^1 \int_0^1 1 + \xi r_1 - \xi r_0 d\mu_0 d\mu_1 \\
    & = \int_0^1 1 + \xi r_1 - \xi \mathbb{E}_{D_R^{[0]}}[R] d\mu_1 \\
    & = 1 + \xi \mathbb{E}_{D_R^{[1]}}[R] - \xi \mathbb{E}_{D_R^{[0]}}[R] \\
    & = 1 + \xi (\mathbb{E}_{D_R^{[1]}}[R] -\mathbb{E}_{D_R^{[0]}}[R]) \\ 
    & \leq 1 \Longleftarrow (\mathbb{E}_{D_R^{[1]}}[R] -\mathbb{E}_{D_R^{[0]}}[R]) \leq 0.
    \end{split}
\end{equation*}
The last term is precisely the definition of $S^-$; therefore, we have shown the nonnegative (super)martingale property holds precisely for any instance in which the null hypothesis is true.\footnote{If we restrict $\xi \in (0, 1)$, then the relation holds bidirectionally, i.e., with $\iff$. } This argument directly extends to any parametric progress metric setting, due to the nature of the increment construction. Thus, it holds for binary, discrete, and continuous valued bounded metrics. 

\subsection{Proof of \Cref{nsm_theorem}}
\label{proof_of_nsm_theorem}
We now utilize the results of \Cref{proof_of_nsm_lemma}, which demonstrated the equivalent of the nonnegative supermartingale (NSM) property of the evidence aggregation \emph{when the true state of the world lies within the set $S^-$}, to prove \Cref{nsm_theorem}. 
\subsubsection{Verifying Ancillary Conditions}
Note that by construction, $\xi_n \in [0, 1]$ and $(r_{1, n} - r_{0, n}) \in [-1, 1]$ implies that 
\begin{equation*}
    \Delta_n = \bigg[1 + \xi_n(r_{1, n} - r_{0, n}) \bigg]\in [0, 2] \text{ w.p. 1}.
\end{equation*}
Therefore,
\begin{equation}
    \label{stochastic_process_equation}
    X_n := X_0\prod_{i=1}^{n} \Delta_i \geq 0 \text{ }\forall n \geq 1. 
\end{equation}
This verifies nonnegativity. 

\subsubsection{Ville's Inequality}
Ville's Inequality is the critical mechanism whereby the expectation of a nonnegative supermartingale process (see \Cref{nsm_definition}) can be linked to right-tailed quantiles of its realized behavior. 
\begin{defn}[Nonnegative Supermartingale (NSM)]
    \label{nsm_definition}
    \text{ }
    
    Consider a discrete-time stochastic process $\{X_n\}_{n\geq 0}$ equipped with the natural filtration $\mathcal{F}_{s} = \{X_i\}_{i=0}^{s-1}$,\footnote{The natural filtration in this context is simply the available information on which a causal algorithm may act. Consistent with this semantic meaning, we note for completeness that $\mathcal{F}_0 = \{\emptyset\}$. } and w.l.o.g. let $X_0 = 1$. The process $\{X_n\}_{n \geq 0}$ is a \textbf{nonnegative supermartingale (NSM)} if it is everywhere nonnegative and contracting in expectation with respect to the filtration: 
    \begin{equation}
    \begin{split}
        \inf_n & \{X_n\}_{n \geq 0} \geq 0 \text{ w.p. 1} \\
        \forall n \text{, }& \mathbb{E}[X_{n} \rvert \mathcal{F}_n] \leq X_{n-1}
    \end{split}
    \end{equation}
\end{defn}
Intuitively, an NSM is a `stable process' in that it is lower bounded by $0$ and contracting in expectation. This stability is the intuitive mechanism from which Ville's Inequality arises. 
\begin{VI}
Let $\{X_n\}_{n \geq 0}$ be a nonnegative supermartingale. Then for any $\alpha \in (0, 1)$, 
\begin{equation}
    \label{villes_inequality}
    \mathbb{P}\bigg[\exists n \in \mathbb{N} : X_n \geq \frac{\mathbb{E}[X_0]}{\alpha}\bigg] \leq \alpha.
\end{equation}
\end{VI}
The critical interpretation of this result is that, for all $n \geq 0$, the $1-\alpha$ quantile of $X_n$ is upper bounded by $\mathbb{E}[X_0]/\alpha$ for any $\alpha \in (0, 1)$. Importantly, this means that the result holds even for optional (i.e., selective) determination of a time of decision.\footnote{In the language of stochastic processes and sequential analysis, the time of decision is often referred to as the ``stopping time'' of the process. Adaptively selecting to stop and decide or to continue collecting data is then referred to as ``optional stopping.''}

\subsubsection{Time-Varying $\xi_n$}
\label{proof_of_time_varying_xi}
We now confirm that the use of time-varying $\xi_n = g(\mathcal{F}_{n-1})$, measurable with respect to the filtration, do not violate the Type-1 Error bounds in \Cref{proof_of_nsm_lemma}. This follows from the definition of the natural filtration and the independence of marginal evaluation outcomes. Specifically, we must modify the proof in \Cref{proof_of_nsm_lemma} to account for the conditional dependencies of $\xi_n = g(\mathcal{F}_{n-1})$ on the previous data. However, the change is minimal due to the independence of $\xi_n$ with the \emph{new data} $(r_{0, n}, r_{1, n})$. We include the modification for completeness: 
\begin{equation*}
    \begin{split}
        \mathbb{E}[1 + \xi_n & (r_{1, n} - r_{0, n})]  = \mathbb{E}[1] + \mathbb{E}[\xi_n(r_{1, n} - r_{0, n})] \\
        & = \mathbb{E}[1] + \mathbb{E}[\xi_n]\mathbb{E}[(r_{1, n} - r_{0, n})] \\
        & = 1 + \mathbb{E}[\xi_n](\mathbb{E}_{D_R^{[1]}}[R] -\mathbb{E}_{D_R^{[0]}}[R]) \\
        & \leq \arg \max \{1, 1 + (\mathbb{E}_{D_R^{[1]}}[R] -\mathbb{E}_{D_R^{[0]}}[R])\} \\
        & \leq 1 \Longleftarrow (\mathbb{E}_{D_R^{[1]}}[R] -\mathbb{E}_{D_R^{[0]}}[R]) \leq 0.
    \end{split}
\end{equation*}
The third line follows from the independence of $\xi_n$ and $(r_{0, n}, r_{1, n})$, and the maximum in the penultimate line arises from taking the extremal values (0 and 1) of $\mathbb{E}[\xi_n]$. We again observe that membership in the the null hypothesis is precisely sufficient to ensure the NSM property. 

\subsubsection{Completing the Proof}
We consider the stochastic process defined in \Cref{stochastic_process_equation} taking $X_0 = 1$. From \Cref{nsm_lemma} and the ancillary verification, we have that $\{X_n\}$ is a nonnegative supermartingale on $S^-$. Using Ville's Inequality, we conclude that, for any possible true state of the world represented by some $h \in S^-$, the probability that $\max_n \{X_n\}$ exceeds $1/\alpha^*$ is less than or equal to $\alpha^*$. Therefore, using the stopping rule defined in \Cref{stopping_rule}, the probability of falsely rejecting any true null $h \in S^-$ is uniformly bounded by $\alpha^*$. This is equivalent to the claim of \Cref{theorem_equation}. 

\subsection{Proof of \Cref{time_varying_xi}}
\label{proof_of_xi_optimization}
We describe in more detail the explicit nonparametric representation of the optimization problem for selecting $\xi_n = g(\mathcal{F}_{n-1})$.
As described in \Cref{the_methodology}, we draw inspiration from kernel density estimation to explicitly model the distribution of outcomes of new evaluation draws. We use a simple version KDE with a preset, uniform binning scheme. That is, we represent the distribution of evaluation scores for each policy $\pi_i$ with $k$ bins partitioning the interval $[0, 1]$. For the case of exact partial credit evaluation with $K$ outcomes, these bins can be chosen to precisely model the true underlying distribution when $k=K$. For nonparametric instances or cases with continuous densities, the choice of $k$ trades off greater accuracy in the representation ($k \uparrow$) against computational burden (which decreases as $k \downarrow$). 

Importantly: when constructing the martingale increments in \Cref{the_evidence_integrator} the exact (possibly continuously-valued) evaluation scores must be used to certify Type-1 Error control. However, no restriction is made with regard to how said data is used to sequentially construct $\xi_n$. Very informally, the algorithm to select the multiplier may `deceive itself' however it likes without violating rigorous validity -- it will simply risk being less efficient. This is precisely the key insight --~\OurMethod~will `pretend' that the data is parametric partial credit (via discretization of the observed performance scores) when choosing $\xi_n$, allowing for efficient optimization. Nonetheless, this discretization of the observed data for selecting $\xi_n$ does not invalidate \Cref{nsm_lemma}; it can only affect the efficiency of the process as measured by time-to-decision. This is in \emph{stark contrast with } $\theta$-SAVI, which requires that the \emph{true underlying distribution} be of a parametric (i.e., partial credit) form. 

As a concrete example of the binning procedure, utilizing eleven bins, we may sort the data by its first two significant digits. This is equivalent to 
\begin{equation}
    \text{bin\_index}(r_{i, n}) = \lfloor 10 \cdot r_{i, n} \rfloor \in \{0, 1, \ldots, 10\}. 
\end{equation}
This can then be seen as a lossy compression of the observed data, where we only represent its approximate value:
\begin{equation}
    \tilde{r}_{i, n} \gets \frac{\lfloor 10 \cdot r_{i, n} \rfloor}{10} \in \{0, 0.1, 0.2, \ldots, 1.0\}.
\end{equation}
The importance of this compression lies in reducing all (highly complex) distributions over general progress metrics to the (parametric) family of categorical distributions over $k$ outcomes, where $k$ is precisely the number of bins. 
Now, we fix the binning procedure to be shared for both policies, and denote the vector of $k$ compressed outcomes to be $\mathbf{c} \in \mathbb{R}_+^k$. We will w.l.o.g. assume henceforth that $\mathbf{c}$ is ordered from least to greatest, and that the elements $\mathbf{c}_i$ are distinct.\footnote{Distinctness is not restrictive; if any set of (semantic) outcomes has the same evaluation score, then they can be `lumped together' into a single composite outcome. The binning procedure itself is assumed to be a deterministic function of the evaluation outcome; therefore, it will always ensure distinction between outcomes it observes. The particular semantic meaning of a score, however, may not be directly observable.} The true underlying distributions over outcome scores are compressed to vectors on the $k$-simplex: 
\begin{equation}
    r_{i, n} \sim \mathcal{D}_R^{[i]} \implies \tilde{r}_{i, n} = \mathbf{c}_j; j \sim \text{Categorical}(\mathbf{p}_i). 
\end{equation}
There are various useful quantities which arise from this representation. The probability of each possible joint evaluation outcome (i.e., of $\pi_0$ and $\pi_1$) can be simultaneously represented as a $k \times k$ square matrix $P$: 
\begin{equation}
    P = \mathbf{p}_0 \mathbf{p}_1^T, 
\end{equation}
where $P_{ij}$ is understood to be `the probability that, for a new evaluation draw, $\pi_0$ gets a (compressed) score $\tilde{r}_{0, n} = \mathbf{c}_i$ and $\pi_1$ gets a (compressed) score $\tilde{r}_{1, n} = \mathbf{c}_j$.' Furthermore, the set of approximate evidence integrator outcomes can be represented by a $k \times k$ square matrix $A$, where: 
\begin{equation}
    A_{ij}(\xi_n; \mathbf{c}) = 1 + \xi_n (\mathbf{c}_j - \mathbf{c}_i). 
\end{equation}
To link this to \Cref{nsm_lemma}: in the special case of discrete partial credit structure, \Cref{nsm_lemma} amounts to demonstrating that the following statement is true \emph{under the null}:
\begin{equation}
    \begin{split}
    \label{null_condition_equation_in_discrete}
    & \sup_{\mathbf{p}_i} \langle A, \mathbf{p}_0\mathbf{p}_1^T\rangle \leq 1 \\ 
    \text{for all } & \hspace{2mm} \mathbf{p}_i \in \Delta_k \\
    & \hspace{2mm} \mathbf{c}^T(\mathbf{p}_1 - \mathbf{p}_0) \leq 0.
    \end{split}
\end{equation}
Conversely, \Cref{nsm_lemma} is sufficient to demonstrate that the above statement must be true, as the latter follows from the generality of the former.\footnote{This can be shown independently using properties of the matrix $A$ and some linear algebraic identities, but is outside the scope of the core intuition. } The key point here is that this is precisely a stability condition on the expectation of the evidence aggregator \emph{when the true state of the world is an element of the null set $S^-$.} The optimization of $\xi_n$, by contrast, relates to optimally \emph{de-stabilizing} the evidence aggregator when the the true state of the world is an element of the alternative set, $S^+$. In that setting, we very much wish for the expectation to be \emph{greater than one}, in contrast with \Cref{null_condition_equation_in_discrete}. 
\subsubsection{Intuition for Optimizing $\xi$}
With the preceding development, our KDE-inspired approach attempts to optimize $\xi$ over the lossy representation induced by the discretization (i.e., the nonparametric distribution representation as, approximately, a discrete partial credit random variable).\footnote{Doing this ever-more efficiently is precisely a subject of future work, as much of the specific domain knowledge of KDE is not present in our simplified implementation.} Unlike in verifying the NSM property (\Cref{nsm_lemma}), the linear algebraic representation of the partial credit problem provides insight into choosing $\xi_n$. Recall that the choice of $\xi_n$ does not affect Type-1 Error, and therefore does not affect any state of the world in which the null is true. Therefore, it will only be used to accelerate detection when the state of the world is such that the alternative is true. The core idea is to observe two phenomena arising out of realizations of $A_{ij}$: the `signal effect' and the `hysteresis effect.' 
\subsubsection{Signal Effect}
The `signal effect' amounts to direct evidence for the alternative. This arises when $r_{1, n} > r_{0, n}$; when this is the case, the multiplier $\Delta_n$ grows with $\xi_n$. Thus, greater likelihood of seeing positive differences in the metrics (`positive signals') promotes \emph{increasing} the value of $\xi$. This is linear in the \emph{asymmetric component} of $P$, in the following sense. Define
\begin{equation}
    \Delta P_{ij} := {P}_{ij} - {P}_{ji}\text{, for } j > i,
\end{equation}
and $0$ otherwise. By definition, $\Delta P$ is an upper triangular matrix. In general, if the matrix has more \emph{positive} elements, this is evidence that the alternative is more likely to be true. That is, $\Delta P_{ij} > 0$ for some $i < j$ means that the probability of observing $r_{1, n} - r_{0, n} = \mathbf{c}_j - \mathbf{c}_i > 0$ is \emph{larger than} the converse of $r_{1, n} - r_{0, n} = \mathbf{c}_i - \mathbf{c}_j < 0$. Considering all pairs $(i, j)$, we observe that $\Delta P$ precisely encodes \emph{asymmetry} in $P$ and its contribution to differences in the mean performance of $\pi_0$ vs $\pi_1$. Considering the expected martingale growth rate, we can observe now that this contributes to the growth linearly in $\xi$: 
\begin{equation}
    \begin{split}
        & \mathbb{E}\bigg[\frac{M_{n+1}}{M_n} \bigg\rvert (\tilde{r}_{0,n}, \tilde{r}_{1,n}) = (i,j)\text{, } j > i\bigg] \\
       & = \bigg((1 + \xi(\mathbf{c}_j - \mathbf{c}_i))\Delta P_{ij}\bigg) \\
       & = 1 + \xi (\mathbf{c}_j - \mathbf{c}_i) \Delta P_{ij}.
    \end{split}
\end{equation}
Therefore, the aggregate positive evidence that $\pi_1$ is better than $\pi_0$ is the sum of these pairwise effects: 
\begin{equation*}
    \sum_{i=1}^{k}\sum_{j=i}^k (1 + \xi(\mathbf{c}_j - \mathbf{c}_i))\Delta P_{ij}. 
\end{equation*}
As noted previously, the key idea of~\OurMethod~is to estimate the quantity $\Delta P$ from the data currently observed (i.e., the filtration $\mathcal{F}_n$) to choose an effective $\xi_n$. 

\subsubsection{Hysteresis Effect}
The signal effect generally pushes $\xi_n$ to be larger; by contrast, there is an opposing mechanism which induces it to shrink. This relates to the \emph{symmetric} component of the $P$ matrix, and is termed the `hysteresis effect.' This effect is so named because of how it manifests: symmetric aspects of $P$ correspond to `self-negating' outcomes (e.g., in which $\tilde{r}_{1, n}-\tilde{r}_{0, n} = -(\tilde{r}_{1, n-1} - \tilde{r}_{0, n-1})$. Direct inspection should convince the reader that difference in empirical performance between the policies has not changed from step $n-2$ to step $n$ (each has observed the same total return since step $n-2$). However, in the course of cycling through zero net change in mean performance difference, the value of $X_n$ has \emph{decreased} from $X_{n-2}$. Achieving both of the converse outcomes (i.e., $(i, j)$ and $(j, i)$) is reflected in the \emph{symmetric component} of $P$: the fraction of realizations of $A_{ij}$ which will be `counteracted' by realizations of $A_{ji}$. As just stated, these pairs of outcomes do not change the empirical gap between the policies, but they \emph{negatively impact} the $X_n$. The idea of losing value (in $X_n$) via a closed loop in net performance difference motivates the term `hysteresis.' Mathematically, we first define 
\begin{equation}
    \underline{P}_{ij} = \min\{P_{ij}, P_{ji}\}.
\end{equation}
This symmetric matrix quantifies the degree to which hysteresis plays a part. The effect on the stochastic process value can be observed via approximate Taylor Expansion (assuming for now that $\xi_n$ is slowly varying): 
\begin{equation}
    \begin{split}
    \label{hysteresis_effect_equation}
        X_{n+2} & = X_n (1 + A_{ij})(1+A_{ji})\underline{P}_{ij} \\
        \implies X_{n+2}/X_n & = (1 + \xi_n(\mathbf{c}_j - \mathbf{c}_i))(1 - \xi_{n+1}(\mathbf{c}_j - \mathbf{c}_i))\underline{P}_{ij} \\
        & \approx (1 + \bar{\xi}(\mathbf{c}_j - \mathbf{c}_i))(1 - \bar{\xi}(\mathbf{c}_j - \mathbf{c}_i))\underline{P}_{ij} \\ 
        & = 1 - \bar{\xi}^2(\mathbf{c}_j - \mathbf{c}_i)^2\underline{P}_{ij}.
    \end{split}
\end{equation}
This term acts to regularize the choice of $\xi_n$, because it suggests that, even with no net gain of information, the stochastic process will tend to decay, and that this decay is larger when $\xi_n$ is larger. Thus, hysteresis motivates a smaller $\xi_n$, opposing the signal effect. However, importantly, unlike the signal effect, the hysteresis effect is quadratic in $\xi$. At an informal level, this is suggestive of maximizing a concave quadratic function over a convex domain, which is a convex optimization problem. Continuing the informal discussion, this suggests that the signal effect (which is linear) will always locally dominate and $\xi_n$ will never be forced to zero when the means differ favorably.  

\begin{rem}[Linearity of $\underline{P}_{ij}$]
    The optimization of $\xi_n$ is implicitly single-step, which should not be suboptimal given the temporal independence of evaluation outcomes. The weighting of the hysteresis terms arises from understanding each component in the single-step context. That is, evaluation outcomes are partitioned as ``an observation $(i, j)$ which will be balanced out by an associated $(j, i)$" (hysteresis) versus ``an observation $(i, j)$ which will \emph{not} be balanced out by an associated $(j, i)$" (signal). Of course, the outcomes which will cancel in the future are twice $\underline{P}_{ij}$ (because one can get either the contributing $(i, j)$ outcome \emph{or} the $(j, i)$ outcome), but the fact that \emph{both} outcomes are required to achieve hysteresis means that each individual observation (that is, $(i, j)$ xor $(j, i)$) should be weighted \emph{by one half}. Thus, the appropriate weighting in the right-hand term in \Cref{hysteresis_effect_equation} is precisely $\underline{P}_{ij}$, as opposed to either $\underline{P}_{ij}^2$ (which is not single-step) or $2\underline{P}_{ij}$, which does not take into account the fact that \emph{both outcomes} are needed for hysteresis to occur.  
\end{rem}
\subsubsection{The Optimization Problem}
With the preceding development, we attempt to maximize the log-value of the stochastic process via single-step optimization, given the currently available information. 

Breaking this down: the log-value of $X_n$ is precisely the sum of the log $\Delta_i$, per the definition of the evidence integrator in \Cref{the_methodology}. The available information manifests as 
\begin{equation}
    \hat{P} = \hat{\mathbf{p}}_0\hat{\mathbf{p}}_1^T,
\end{equation}
which represents the empirical distribution of the observed results. Each $\hat{\mathbf{p}}_i$ can be computed by simply calculating the empirical frequency of occurrence of each bin.
Thus, the approximate expected multiplier value (given the current available information) is $\langle A, \hat{P} \rangle$. \Cref{nsm_lemma} guarantees that this expectation is bounded by 1 when the null is true; we are interested, however, precisely in the case where the alternative is true. In that setting, the bound does not apply. 

Due to the linearity of the inner product, the optimization essentially acts elementwise on $\hat{P}$-weighted elements of $\log{A_{ij}}$: 
\begin{equation*}
    \label{eqn:full_optimized_xi}
    \begin{split}
        \xi^* & = \arg \max_{\xi \in [0, 1)} \bigg(\sum_{i=1}^{K-1} \sum_{j=i+1}^K \lvert \Delta \hat{P}_{ij}\rvert \log{\big(1 + \xi \text{ sgn}(\Delta \hat{P}_{ij}) \Delta c_{ij}\big)} \\
        & \hspace{15mm} \ldots + \underline{\hat{P}}_{ij}\log{\big(1 - \xi^2 \Delta c_{ij}^2\big)}\bigg). \\
    \end{split}
\end{equation*}

Applying first-order optimality conditions, we obtain an efficient representation for which a root-finding procedure on $[0, 1]$ quickly converges:
\begin{equation*}
    \label{eqn:FOCs}
    \begin{split}
        0 & = \nabla_\xi (\ldots)\rvert_{\xi^*} \\
        & =  \sum_{i=1}^{K-1} \sum_{j=i+1}^K \frac{\Delta \hat{P}_{ij}\Delta c_{ij}}{1 + \xi^* \text{ sgn}(\Delta \hat{P}_{ij})\Delta c_{ij}} - 2\xi^* \frac{\underline{\hat{P}}_{ij}\Delta c_{ij}^2}{1 - (\xi^*)^2 \Delta c_{ij}^2}.
    \end{split}
\end{equation*}
This can be understood as choosing the optimal $\xi_n$ to maximize the expected growth rate of $X_n$ under the currently available estimate of the distributions $\mathbf{\hat{p}}_i$. This choice explicitly balances the signal and hysteresis effects in order to achieve this maximization. Further, though it is beyond the scope of the present work, it is believed that the nominal objective is in fact concave (via analysis of the second-order shape of the objective). Thus, the current approach is likely sufficient despite only the first-order verification; additionally, faster approaches are likely feasible. Regardless, this proves the necessary result, and describes the practical optimization scheme used for~\OurMethod~in all experiments. 
\section{Additional Exposition of Baseline Methods}
\label{additional_baseline_information}
\subsection{Waudby-Smith Ramdas Confidence Sequences}
The Waudby-Smith-Ramdas (WSR) procedure~\cite{waudby2024estimating} provides time-uniform, non-parametric, and non-asymptotic confidence sequences for mean estimation. Suppose we wish to estimate the mean \(\mu^*\) of a bounded random variable \(Z\) using a sequential stream of observations \(Z_1, Z_2, \ldots\). A confidence sequence is a sequence of confidence intervals $\{CI_t\}_{t=1}^{\infty}$ such that $$\mathbb{P}(\forall t\geq 0,\, \mu^* \in CI_t) \geq 1-\alpha,$$ where \(\alpha\) is a pre-specified error budget and \(CI_t = [l_t, u_t] \subseteq \mathbb{R}\) are intervals. ~\Cref{alg:meanCS} gives an overview of confidence sequence estimation. The WSR procedure typically provides tighter bounds in the same number of samples than other non-parametric concentration methods such as Hoeffding~\cite{hoeffding1963probability} and empirical Bernstein bounds~\cite{maurer2009empirical}. We apply WSR to sequential policy comparison by constructing confidence sequences on the difference in  evaluation scores of the two policies. That is, we estimate the mean of the sequence $Z_n = (r_{1,n} - r_{0,n})$, which denotes the difference in evaluation scores. If the time-uniform sequence ever excludes $0$, then there is strong evidence that the policies are significantly different from each other. Specifically, if the interval is entirely negative, it is evidence for the null; if positive, evidence for the alternative. The test is thus formed by constructing the confidence sequence and stopping precisely at $\arg \min_{t\in \mathbb{N}} \{t: 0 \notin CI_t\}$.

\begin{algorithm}[h]
\caption{Waudby-Smith--Ramdas (WSR) procedure for Confidence Sequences~\cite{waudby2024estimating}}
\label{alg:meanCS}
\begin{algorithmic}[1]
\REQUIRE Data $\{Z_1,\dots,Z_n\}$, error level $\alpha \in (0,1)$, range $[L,U]$ such that $Z_i \in [L,U]$
\ENSURE Confidence sequence $CS$ for the mean
\STATE Confidence Sequence $CS \gets \{\}$
\FOR{$i \gets 1$ \TO $n$}
    \STATE $Z_i \gets (Z_i - L)/(U-L)$ \hfill{\small // Normalize to $[0,1]$}
\ENDFOR
\STATE $c=0.95$ \hfill{\small // Hyperparameter for computing martingales}
\STATE Construct fine grid $M_{\rm grid}$ over $[0,1]$
\STATE Initialize set of candidate means $\mathcal{A} \gets M_{\rm grid}$

\FOR{$t \gets 1$ \TO $n$}
    \STATE $\hat{\mu}_t \gets  (0.5 + \sum_{j=1..t} Z_j)/(t+1)$
    \STATE $\hat{\sigma}_t^2 \gets \big(0.25 + \sum_{j=1}^t (Z_j - \hat{\mu}_t)^2\big)/\big(t+1\big)$
    \STATE $\lambda_t \gets \sqrt{\big(2\log(2/\alpha)\big)\big(t \log(t+1)\,\hat{\sigma}_{t-1}^2\big)}$ {\small // Betting coefficient}
    \FORALL{$m \in M_{\rm grid}$}
        
        \STATE $M_t^+(m) \gets \Bigl(1+\min\bigl(\lambda_t, \tfrac{c}{m}\bigr)(Z_t - m)\Bigr)\, M^+_{t-1}(m)$
        \STATE $M_t^-(m) \gets \Bigl(1-\min\bigl(\lambda_t, \tfrac{c}{1-m}\bigr)(Z_t - m)\Bigr)\, M^-_{t-1}(m)$
        \STATE $M_t(m) \gets \tfrac{1}{2}\max\{M_t^+(m),\, M_t^-(m)\}$ \hfill{\small // Martingale}

        \IF{$M_t(m) \geq 1/\alpha$}
            \STATE $\mathcal{A} \gets \mathcal{A} \setminus \{m\}$ \hfill{\small // Remove $m$ from candidate means}
        \ENDIF
    \ENDFOR
    \STATE $C_\alpha = \{m(U-L)+L : m \in \mathcal{A}\}$ \hfill{\small // True mean lies in this set w.h.p.}
\STATE $CI_t = \bigl[\max\{0,\,\min C_\alpha\},\ \min\{1,\,\max C_\alpha\}\bigr]$
\STATE Append $CI_t$ to $CS$
\ENDFOR
\RETURN $CS$
\end{algorithmic}
\end{algorithm}

\begin{table*}[h]
    \centering
    \resizebox{\linewidth}{!}{%
    \begin{tabular}{lcc|cc|cc|cc|cc|cc}
    \toprule
    \hfill Comparison $\rightarrow$ & \multicolumn{2}{c|}{\textbf{PPO vs. TD3}} & \multicolumn{2}{c|}{\textbf{PPO vs. DDPG}} & \multicolumn{2}{c|}{\textbf{PPO vs. SAC}} & \multicolumn{2}{c|}{\textbf{SAC vs. DDPG}} & \multicolumn{2}{c|}{\textbf{TD3 vs. DDPG}} & \multicolumn{2}{c|}{\textbf{SAC vs. TD3}} \\
    Task $\downarrow$ \hfill Method $\rightarrow$ & WSR &~\OurMethod & WSR &~\OurMethod & WSR &~\OurMethod & WSR &~\OurMethod & WSR &~\OurMethod & WSR &~\OurMethod \\
    \midrule
    Ant-v4                  & 69 & 60 & 27 & 21 & 13 & 12 & 10 & 9 & 21 & 19 & 14 & 13 \\
    HalfCheetah-v4          & 10 & 9 & 8 & 8 & 8 & 7 & 40 & 38 & 26 & 23 & 18 & 15 \\
    Hopper-v4               & 142 & 79 & 13 & 12 & 46 & 41 & 17 & 14 & 11 & 10 & 30 & 20 \\
    InvertedPendulum-v4     & 39 & 39 & 677 & 267 & 39 & 39 & 46 & 46 & 46 & 46 & -- & -- \\
    Humanoid-v4             & 8 & 7 & 42 & 40 & 8 & 7 & 8 & 8 & 8 & 8 & 96 & 89 \\
    Walker2d-v4             & 25 & 25 & 23 & 22 & 12 & 11 & 8 & 7 & 11 & 10 & 24 & 22 \\
    Pusher-v4               & 70 & 61 & 89 & 77 & 56 & 47 & 133 & 109 & 294 & 237 & 249 & 220 \\
    \midrule
    Total (14000 nominal)   & 726 & \textbf{560} & 1758 & \textbf{894} & 364 & \textbf{328} & 524 & \textbf{462} & 834 & \textbf{706} & 2862 & \textbf{2758} \\
    \bottomrule
    \end{tabular}
    }
    \caption{\textbf{Empirical time-to-decision for all \textit{reinforcement learning} pairwise policy comparisons on Mujoco benchmark tasks using WSR and~\OurMethod.} If a decision is not reached, the entry is left blank; for the purpose of computing evaluation savings, any blank entry is counted at $N$ trials. All simulated tasks utilize $N = 1000$. Thus, as there are two algorithms being compared pairwise, the total number of batch trials for each column is $14$k. Due to not being optimized for policy comparison, WSR requires more trials to reach a decision, uniformly across tasks and learning algorithms. Though performance is similar in the low-variance regime (small TTD, or `easy problems'), there are multiple instances where significant variation in time-to-decision arises due to this suboptimal multiplier selection, most notably with \textit{PPO vs DDPG} on InvertedPendulum-v4. }
    \label{table_of_rl_data_all}
\end{table*}
\section{Additional Experimental Results}
\subsection{Real Valued Continuous Metrics (RL Baselines)}
The following experiments demonstrate that we can apply~\OurMethod~to compare rich behavioral properties of robot policies quantified by continuous metrics such as jitter, smoothness, and stability, that are not easily represented in discrete, partial scoring rubrics. 
We use~\OurMethod~to differentiate between RL policies trained via different canonical algorithms on Mujoco benchmark tasks. Variation in reward structure and order of magnitude is significant across tasks, and the rewards themselves are complex, often continuous, and span multiple orders of magnitude at convergence to benchmark performance (see \Cref{tab:mujoco_rewards}).\footnote{For additional details on the structure, see, e.g., the implementation of \cite{huang2022cleanrl}.} Due to this complexity, the WSR procedure is the only applicable baseline, as the distribution over evaluation rewards is nonparametric. 
\begin{figure}
    \centering
    \includegraphics[width=\linewidth]{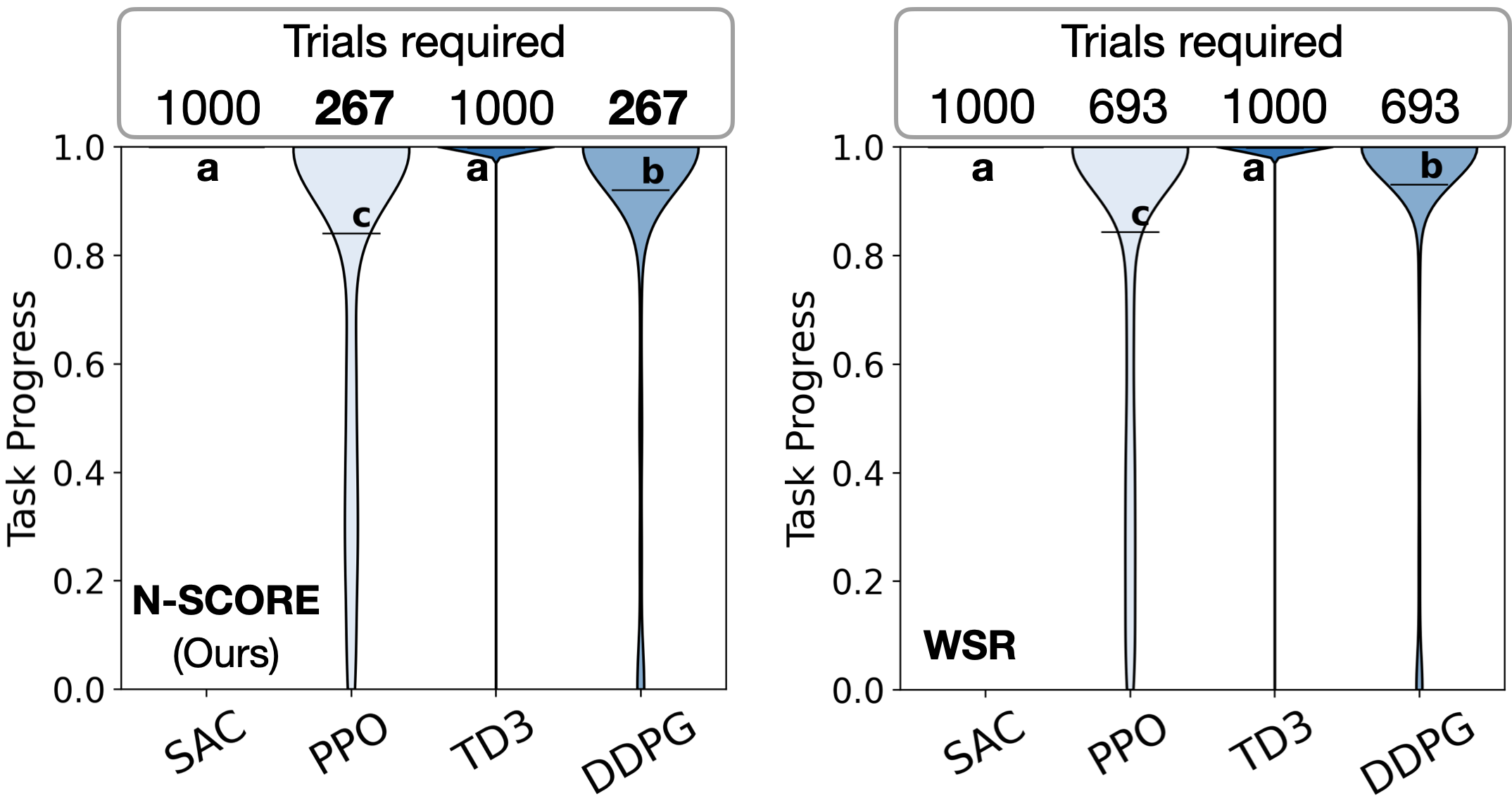}
    \caption{\textbf{Number of trials required to separate RL policies on InvertedPendulum-v4.} This is one of the most stark instances of disparity between~\OurMethod~and WSR, the latter of which requires significantly more evaluation effort (by nearly a factor of three) to reach a decision. }
    \label{fig:inv_pend}
\end{figure}
We train RL policies on Mujoco benchmarks using the CleanRL library~\cite{huang2022cleanrl}. Policies are trained for $10^6$ global steps. Mean policy rewards are averaged over $1000$ independent episodes using a frozen policy post-training.~\Cref{tab:mujoco_rewards} lists the average episodic return of the various policies, which are comparable to SOTA training curves listed in Spinning Up RL baselines~\cite{SpinningUp2018}.~\Cref{fig:2x3} shows the the trials required for each benchmark to separate all RL policies in a multi-policy comparison setting where the global error tolerance \(\alpha=0.05\) is split between six pairwise comparisons.~\Cref{fig:inv_pend} illustrates the time to separation for InvertedPendulum-v4. In this case, both~\OurMethod~and WSR cannot distinguish between SAC and TD3, which have an average empirical performance of 1000 and 998, respectively. Both policies attain the maximum possible reward for this benchmark, making it statistically impossible to distinguish in a 1000 trials.~\Cref{table_of_rl_data_all} lists the time-to-decision on pairwise comparison of policies, showing significant savings of~\OurMethod~ over WSR in the number of trials required, most notably saving over 800 trials in the comparison of DDPG and PPO. In total, from a nominal simulation burden of 28k batch evaluation rollouts,~\OurMethod~ requires approximately 5600 evaluations to make all necessary comparisons (savings of over 80\%), while WSR requires approximately 7100 evaluations. These each reflect strong savings over batch methods, while~\OurMethod~displays an additional 20\% improvement in total sample complexity over the entire RL evaluation process. 
\begin{figure*}
    \centering
    \begin{subfigure}{0.48\linewidth}
        \centering
        \includegraphics[width=\linewidth]{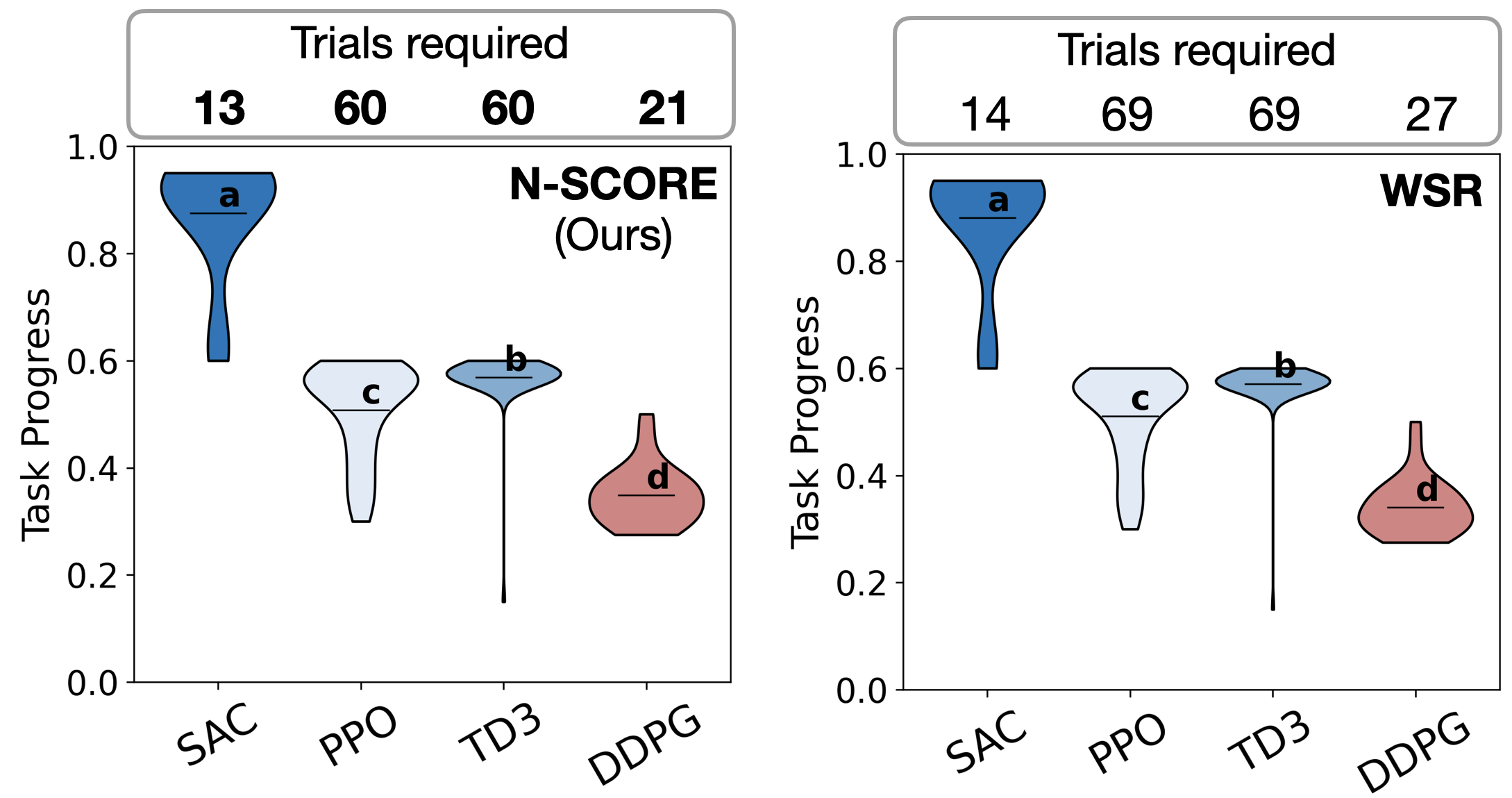}
        \caption{Ant}
    \end{subfigure}\hfill
    \begin{subfigure}{0.49\linewidth}
        \centering
        \includegraphics[width=\linewidth]{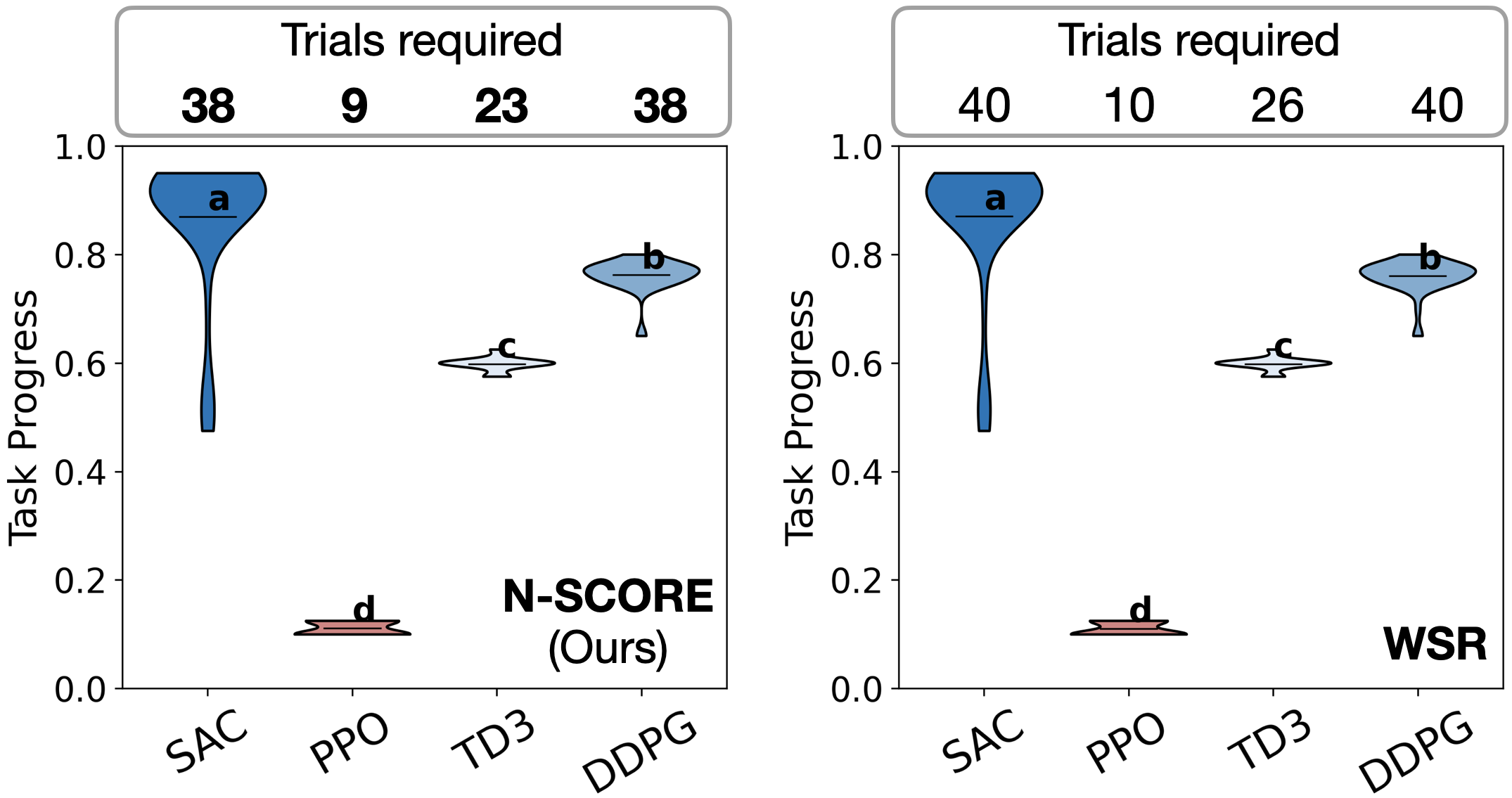}
        \caption{Half Cheetah}
    \end{subfigure}
    \vspace{1em}
    
    \begin{subfigure}{0.48\linewidth}
        \centering
        \includegraphics[width=\linewidth]{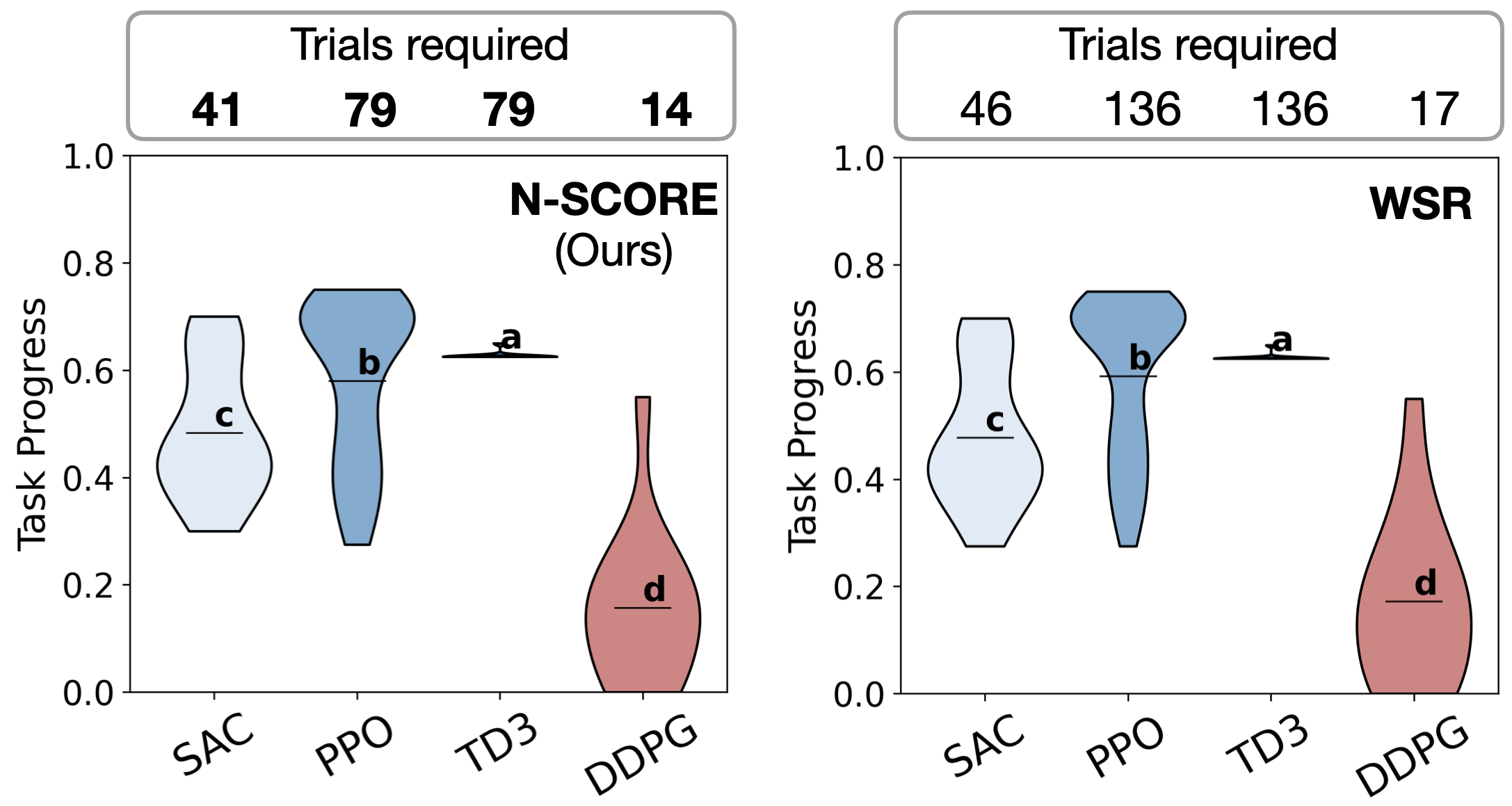}
        \caption{Hopper}
    \end{subfigure}
    \hfill
    \begin{subfigure}{0.48\linewidth}
        \centering
        \includegraphics[width=\linewidth]{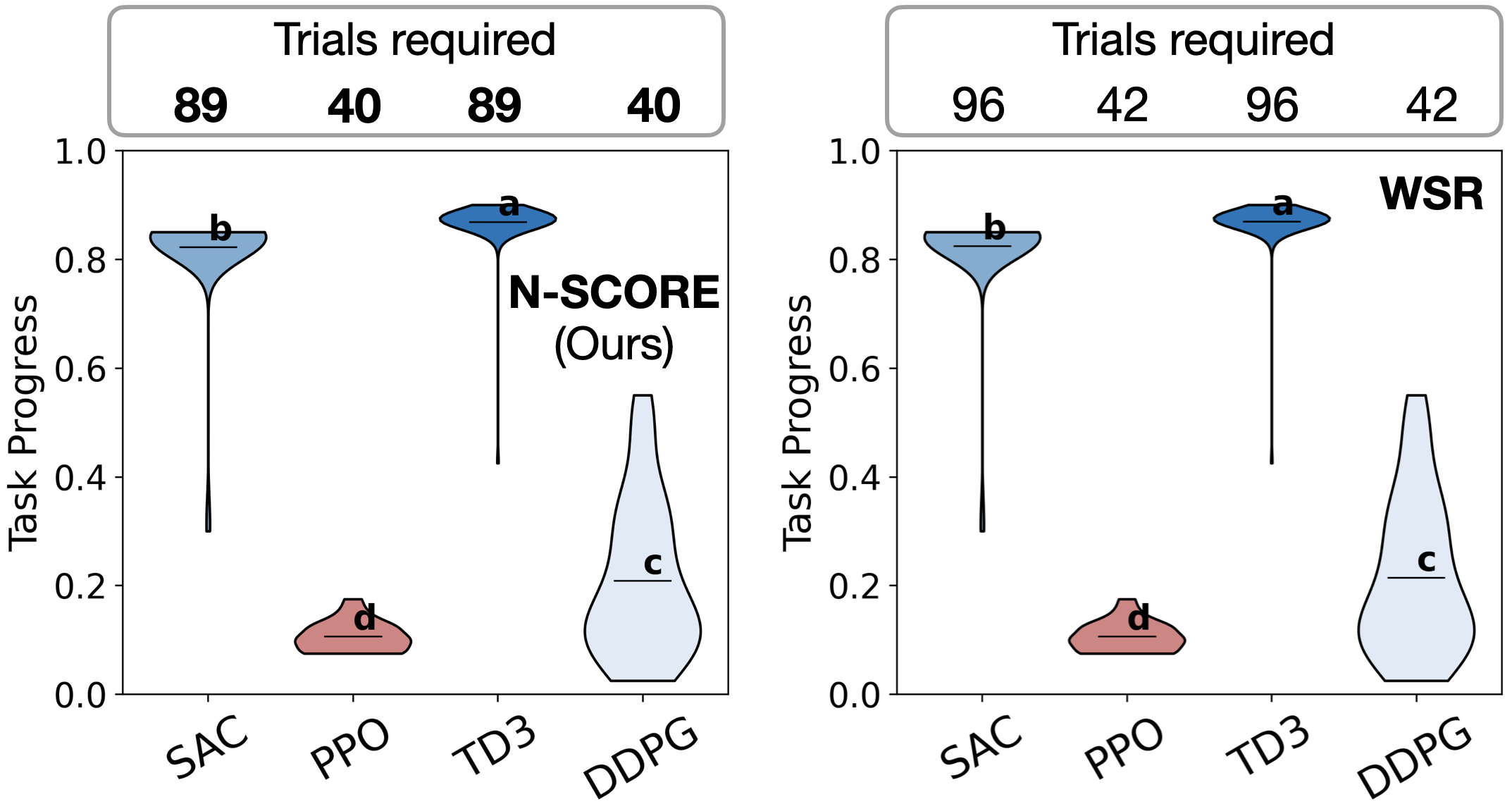}
        \caption{Humanoid}
    \end{subfigure}
    \vspace{1em}
    \begin{subfigure}{0.48\linewidth}
        \centering
        \includegraphics[width=\linewidth]{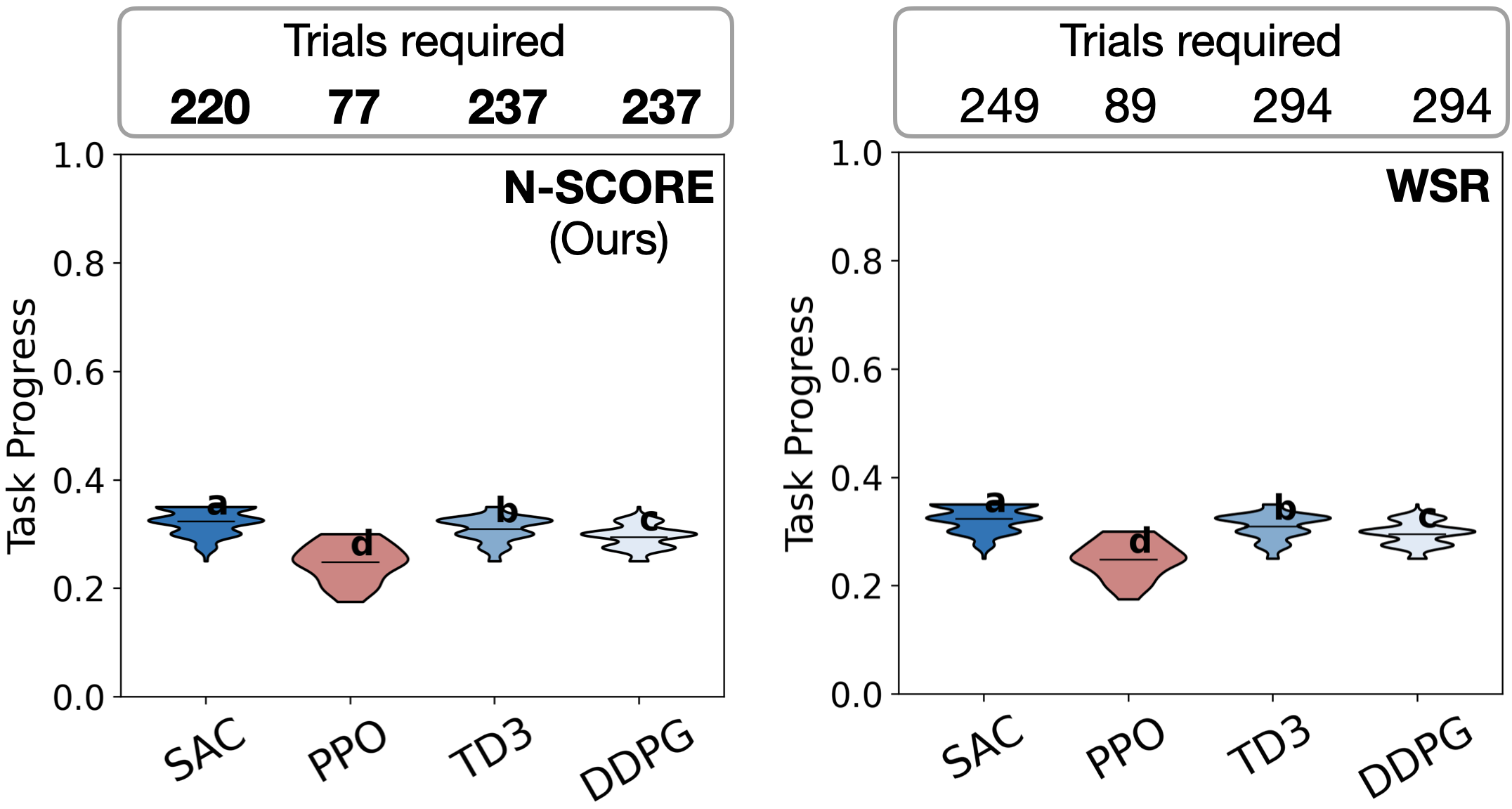}
        \caption{Pusher}
    \end{subfigure}\hfill
    \begin{subfigure}{0.48\linewidth}
        \centering
        \includegraphics[width=\linewidth]{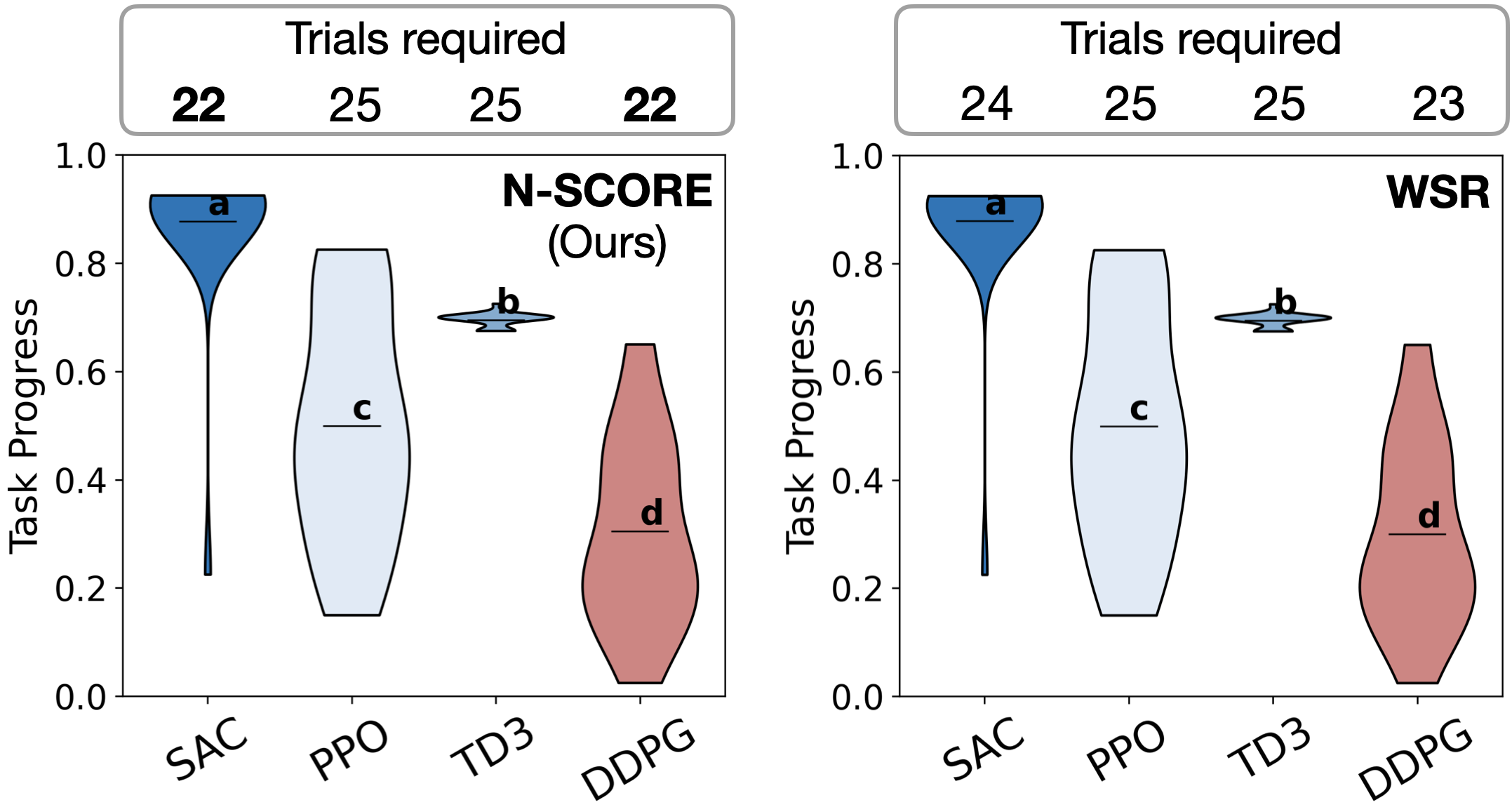}
        \caption{Walker2D}
    \end{subfigure}

    \caption{\textbf{Violin plots and the number of samples required for~\OurMethod~ and WSR on multi-policy comparison of RL policies on Mujoco benchmarks.} Policies with different letters are statistically distinguishable by the method. Policies are compared at a global error bound of $\alpha = 0.05$ with a Bonferroni correction. In all cases,~\OurMethod~results in the same comparison conclusions as WSR with fewer samples, demonstrating its broadly improved efficiency. These results also serve as an alternate visualization of the time-to-decision results in \Cref{table_of_rl_data_all}.}
    \label{fig:2x3}
\end{figure*}

\begin{table}[!htbp]
    \centering
    \resizebox{\columnwidth}{!}{
\begin{tabular}{lrrrr}
\toprule
Task & DDPG & TD3 & PPO & SAC \\
\midrule
Ant-v4 & 418.8 & 2477.5 & 1849.0 & 4901.8 \\
HalfCheetah-v4 & 9932.2 & 7782.5 & 1332.2 & 11759.0 \\
Hopper-v4 & 1106.4 & 3234.6 & 2797.4 & 2488.7 \\
InvertedPendulum-v4 & 927.2 & 998.6 & 848.2 & 1000.0 \\
Pusher-v4 & -38.5 & -35.7 & -47.6 & -32.4 \\
Walker2d-v4 & 1732.5 & 3513.4 & 1679.8 & 4544.1 \\
Humanoid-v4 & 1386.0 & 5288.2 & 748.4 & 5042.7 \\
\bottomrule
\end{tabular}
}
\caption{\textbf{Empirical mean episodic return on evaluation instances for Mujoco benchmark tasks.} Note, in tandem with \Cref{fig:2x3}, that the distributions over rewards near optimality vary significantly in shape and scale. This emphasizes the generality of nonparametric procedures and their broad applicability to creative or nonstandard metrics of robot performance or behavior.}
\label{tab:mujoco_rewards}
\end{table}

\subsection{Binary vs. Continuous Valued Metrics}
\begin{figure*}
    \centering
    \includegraphics[width=\linewidth]{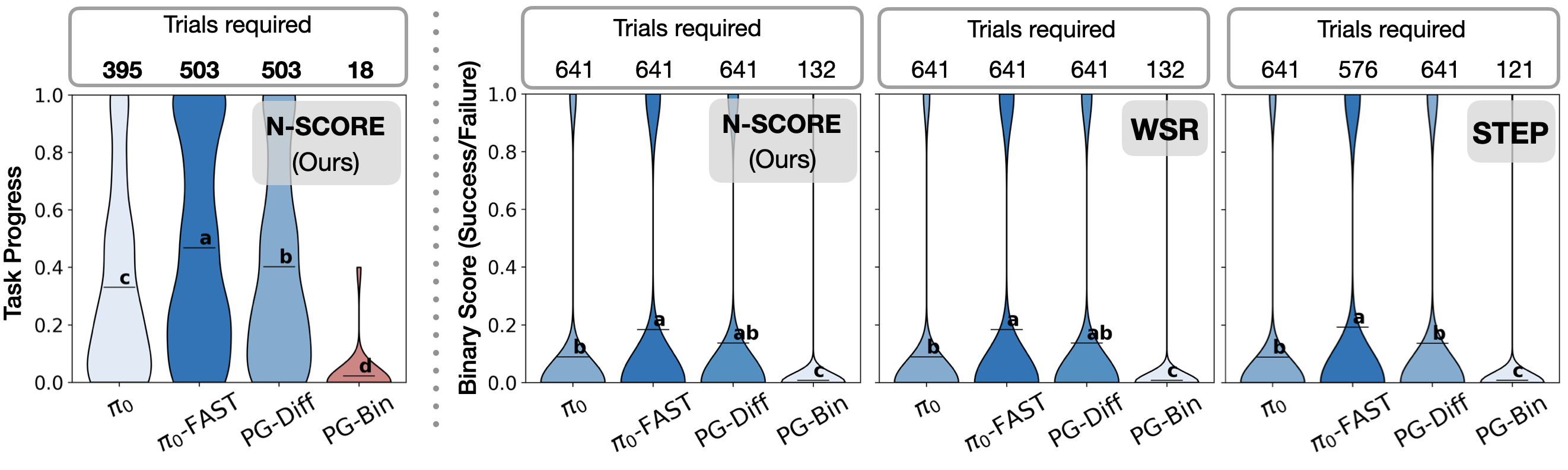}
    \caption{\textbf{Continuous task progress scores enable faster policy comparison on RoboArena policy evaluation data}. Violin plots and time-to-decision for multi-policy comparison from RoboArena evaluations. \emph{Left:} Policy comparison under continuous scores with~\OurMethod. \emph{Right:} policy comparison under Bernoulli scores with~\OurMethod, WSR, and STEP. In the Bernoulli comparison setting, none of the methods are able to distinguish all the policies. Even though STEP is maximally efficient on Bernoulli comparisons,~\OurMethod~on continuous progress requires fewer total trials to distinguish policies. We emphasize that the partial credit and binary metrics arise \emph{from the same rollouts} of each policy; the partial credit thus reflects precisely a more informative `representation' of the rollout for evaluation purposes. The reduced time-to-decision and increased power (separating all of the policies successfully) highlights the fundamental advantage of fine-grained task progress scores over sparse binary success rates for efficient policy comparison. As can be observed in \Cref{fig:roboarena} in the main text,~\OurMethod~also significantly reduces the time-to-decision with respect to WSR. The former requires approximately 1420 trials, while the latter needs an additional 450, while failing to distinguish \textbf{$\pi_0$} from \textit{PG-Diff}. }
    \label{fig:roboarena_binary}
\end{figure*}
We add further validation to the empirical observation of \Cref{the_experiments_section} -- that partial credit evaluation tends to improve average time-to-decision on practical evaluation problems -- by comparing the times-to-decision of STEP \cite{snyder2025your} on binary success data from \cite{atreya2025roboarena} with times-to-decision of~\OurMethod~on equivalent partial credit data taken from the same evaluations. We emphasize that the evaluation burden of using partial credit measures is essentially identical to success measures, so the added informativity (and reduction in necessary evaluation trials), can be realized for free. 

To illustrate this, we rerun multi-policy comparisons in the setting of \cite{atreya2025roboarena}, but using instead the associated binary evaluation metrics. The resulting times-to-decision are shown in \Cref{fig:roboarena_binary}. STEP is maximally sample efficient on Bernoulli scores at 1979 total evaluation trials, but fails to separate two of the policies. In contrast, the total sample complexity for~\OurMethod~on continuous progress scores is 1419, a savings of 560 evaluation trials. Furthermore,~\OurMethod~achieves these savings with higher empirical power, correctly separating every policy (i.e., giving a statistically significant \emph{complete ordering} of the performance of the policies). 

To further illustrate the time-to-decision of a test procedure, we present the anytime-valid p-values against the number of trials, in both continuous and binary metrics. While both WSR and~\OurMethod~are sequential procedures that carry a notion of an anytime-valid p-value, the STEP procedure does not; when STEP rejects the null hypothesis, it is due to the p-value dropping below the error threshold for the comparison.~\Cref{fig:pvalue_binary_vs_real,fig:pvalue} plot the anytime-valid p-value against the number of trials in the pairwise policy comparisons from RoboArena data under both continuous progress and binary task success settings. Due to multi-test correction, the total error budget of \(\alpha=0.05\) is split between the six pairwise comparisons, resulting in an allocation of \(\alpha=0.0083\) each. A test procedure terminates when the p-value is at or below the error threshold. In~\Cref{fig:pvalue_binary_vs_real}, we illustrate time-to-decision under large gaps in the empirical performance of policies. Here, we compare policies to PG-Binning-DROID policy which has the lowest empirical performance in RoboArena, up to 30\(\%\) points lower than other policies on continuous task progress. For both NSM and~\OurMethod, the p-value quickly approaches the error threshold in 18 trials, consistent with the reported time-to-decision illustrated in~\Cref{fig:roboarena}. However, in the Bernoulli score setting, the empirical performance gap shrinks (as shown in~\Cref{fig:roboarena_binary}), making it harder to distinguish and requiring more trials than the continuous progress setting. Furthermore, even when the Bernoulli performance gap is substantial, such as in the case of \(\pi_0\)-FAST-DROID vs. PG-Bin-DROID which differ by around \(20\%\) points, the p-values of the tests require around 80 trials to gather sufficient statistical evidence of distinction (see~\Cref{fig:sub_gap}). Intuitively, for the same performance gap, the signal-to-noise ratio is relatively higher in continuous progress settings because the empirical variance of continuous scores is no greater than the empirical variance of the corresponding Bernoulli scores. These results illustrate that ~\OurMethod~and WSR are able to leverage the rich information captured in continuous evaluation scores to distinguish policies faster than Bernoulli scores. 

In~\Cref{fig:pvalue}, we plot the progression of p-values for comparisons where the performance gaps between policies are small (around \(10\%\) or lower). ~\OurMethod~in the continuous task progress setting is the fastest in distinguishing the policies and WSR in the binary metric setting is the slowest. In the comparison shown in~\Cref{fig:pi0_vs_pi0fast}, ~\OurMethod~even in the binary metric setting requires fewer trials to distinguish than WSR with continuous scores. In the comparisons shown in~\Cref{fig:pi0_vs_pgdiff,fig:pi0fast_vs_pgdiff}, while both ~\OurMethod~and WSR with binary metrics fail to distinguish policies within 641 trials, the p-value of~\OurMethod~is much closer to the \(\alpha\)-threshold than WSR. For a valid multi-policy comparison, we split the error budget and apply a Bonferroni correction to control for type-I error across all pairwise comparisons at the cost of reduced statistical power. In future work, we plan to investigate efficient multi-test correction methods to proportionally allocate the risk budget. This would be complement our current work, enabling further efficiency gains.

Finally, we make an additional observation on the trends of anytime-valid p-values of WSR and~\OurMethod. As seen in~\Cref{fig:pvalue,fig:pvalue_binary_vs_real}, WSR p-values under either metric can show a steeper decrease initially, but p-values of~\OurMethod~ ultimately approaches the risk threshold first. We attribute this efficiency to our optimization of the betting coefficient \(\xi_n\) for policy comparison. In contrast, the betting coefficients for WSR (line 11 of~\Cref{alg:meanCS}) are not optimized for policy comparison. 

\begin{figure*}[h]
    \centering
    \begin{subfigure}{0.33\linewidth}
        \centering
        \includegraphics[width=\linewidth]{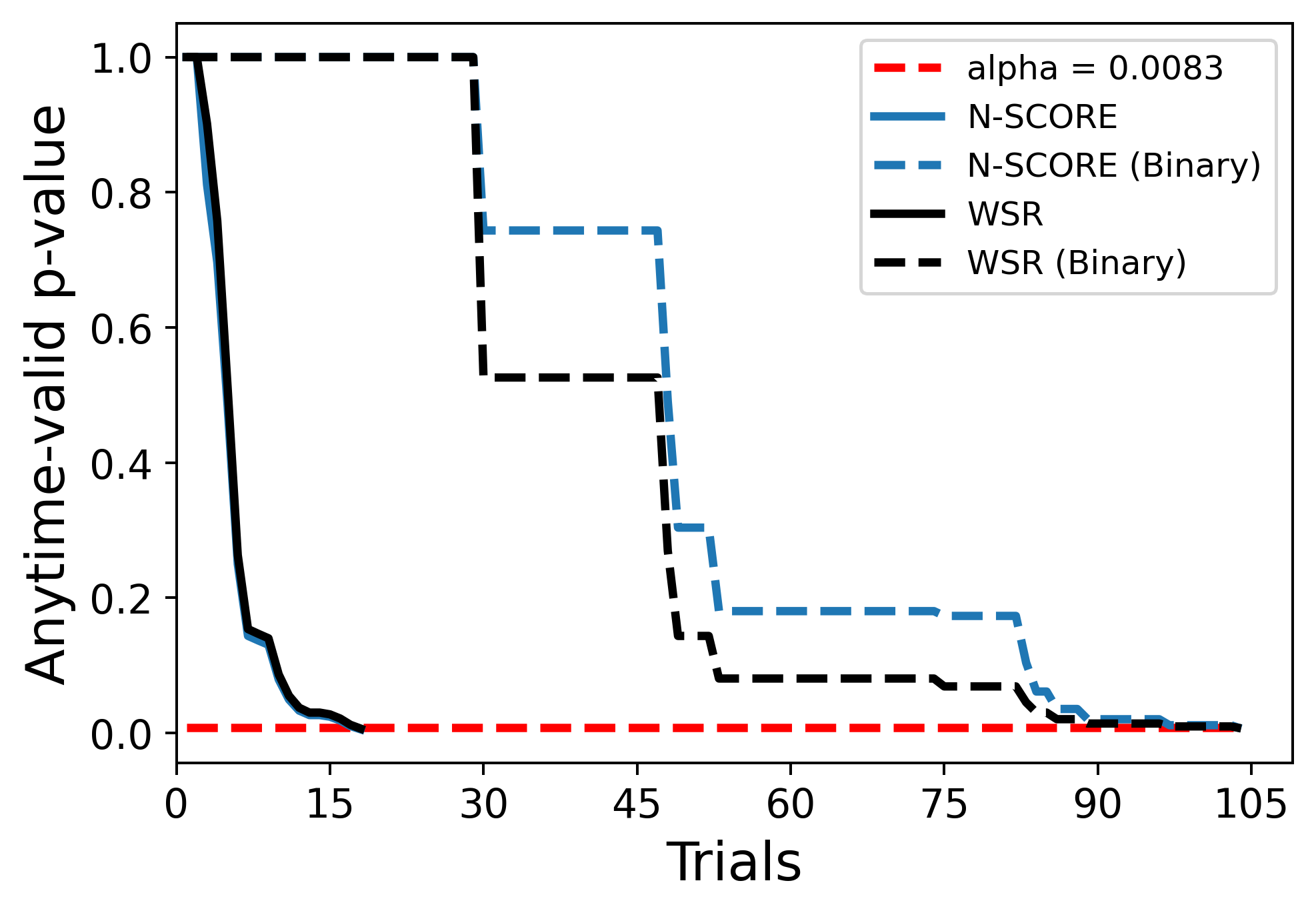}
        \caption{PG-Diff-DROID vs. PG-Bin-DROID}
    \end{subfigure}\hfill
    \begin{subfigure}{0.33\linewidth}
        \centering
        \includegraphics[width=\linewidth]{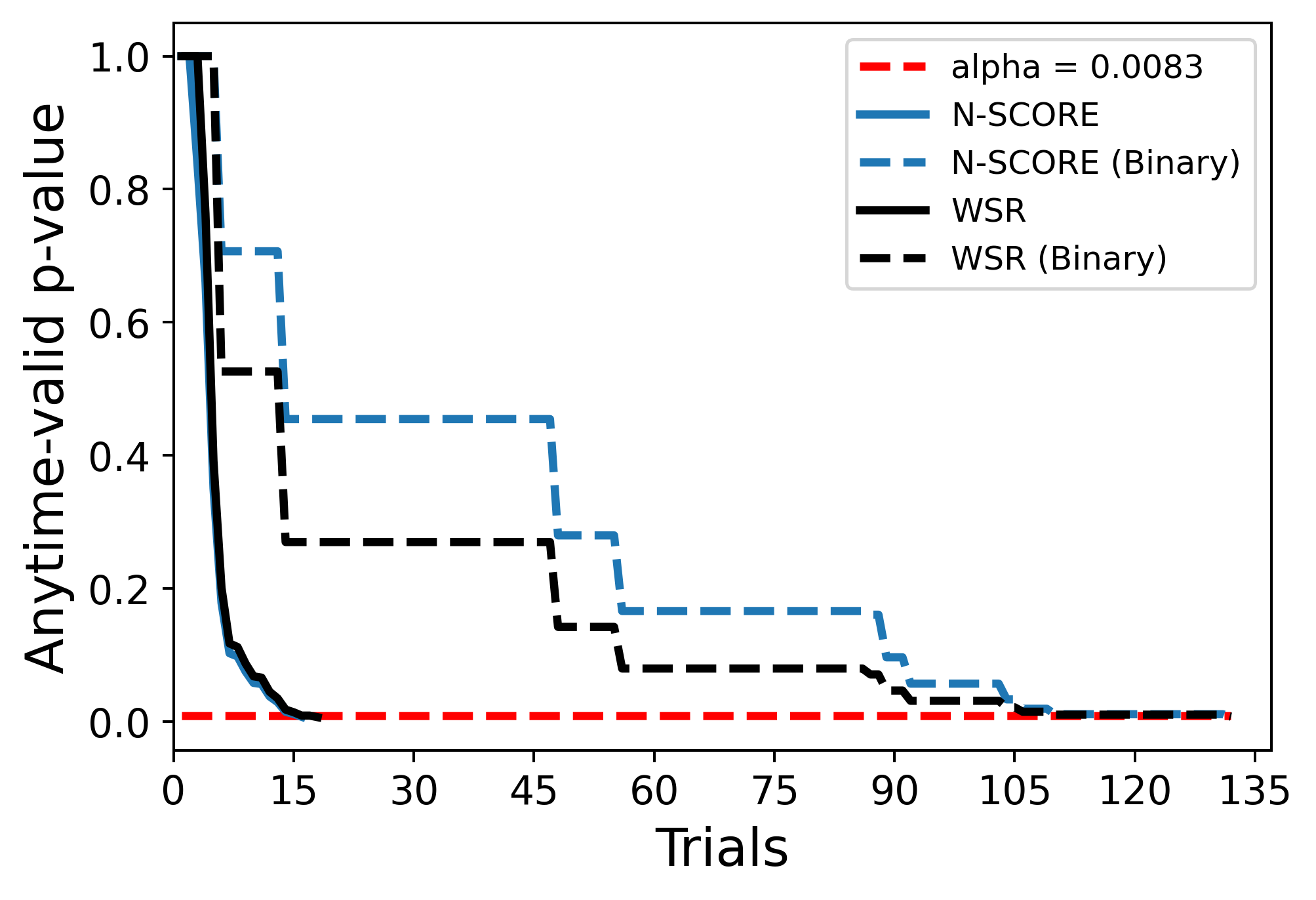}
        \caption{\(\pi_0\)-DROID vs. PG-Bin-DROID}
    \end{subfigure} \hfill 
    \begin{subfigure}{0.33\linewidth}
        \centering
        \includegraphics[width=\linewidth]{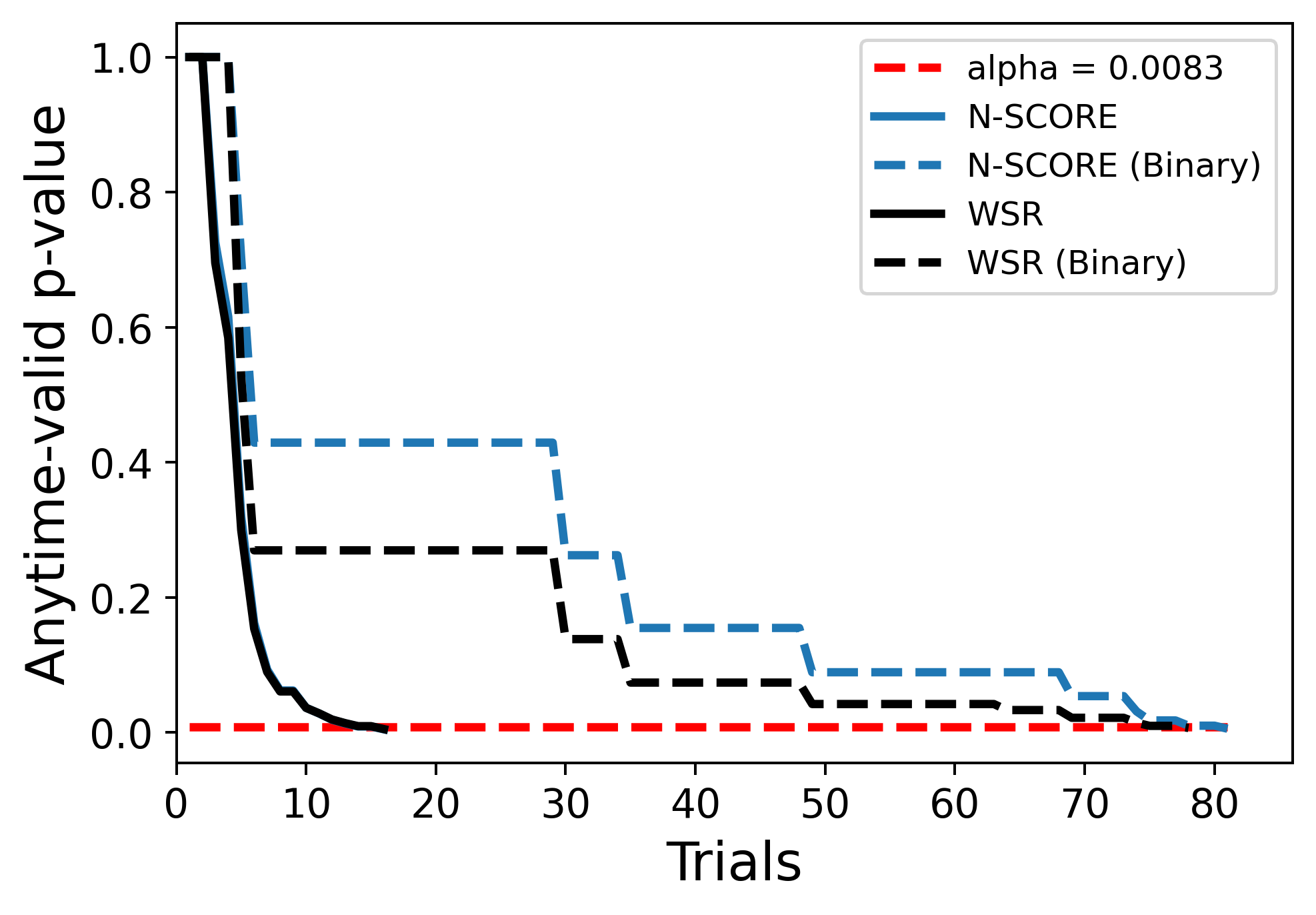}
        \caption{\(\pi_0\)-FAST-DROID vs. PG-Bin-DROID}
        \label{fig:sub_gap}
    \end{subfigure} 
    \caption{\textbf{Anytime-valid p-values vs. number of trials on multi-policy comparisons on RoboArena. A test procedure terminates when its p-value is at or below the error threshold \(\alpha\). As can be seen, it is faster to distinguish policies under task progress metrics as opposed to binary metrics when the policy performance gap is large.} PG-Binning-DROID has an empirical mean that is more than 30\(\%\) points lower than \(\pi_0\)-DROID, \(\pi_0\)-FAST-DROID, and PG-Diff-DROID.}
    \label{fig:pvalue_binary_vs_real}
\end{figure*}

\begin{figure*}[h]
    \begin{subfigure}{0.33\linewidth}
        \centering
        \includegraphics[width=\linewidth]{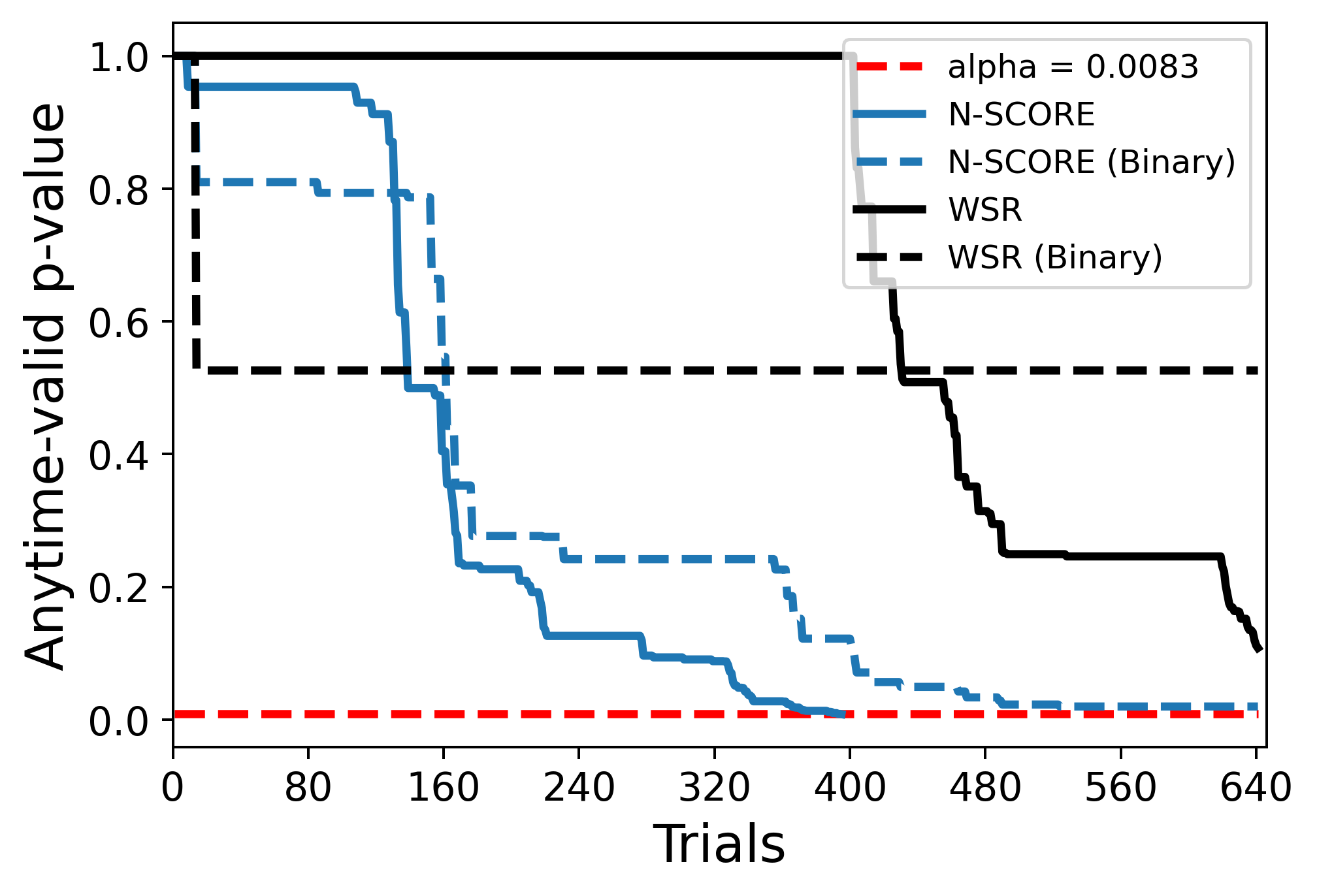}
        \caption{\(\pi_0\)-DROID vs. PG-Diff-DROID}
        \label{fig:pi0_vs_pgdiff}
    \end{subfigure} \hfill
    \begin{subfigure}{0.33\linewidth}
        \centering
        \includegraphics[width=\linewidth]{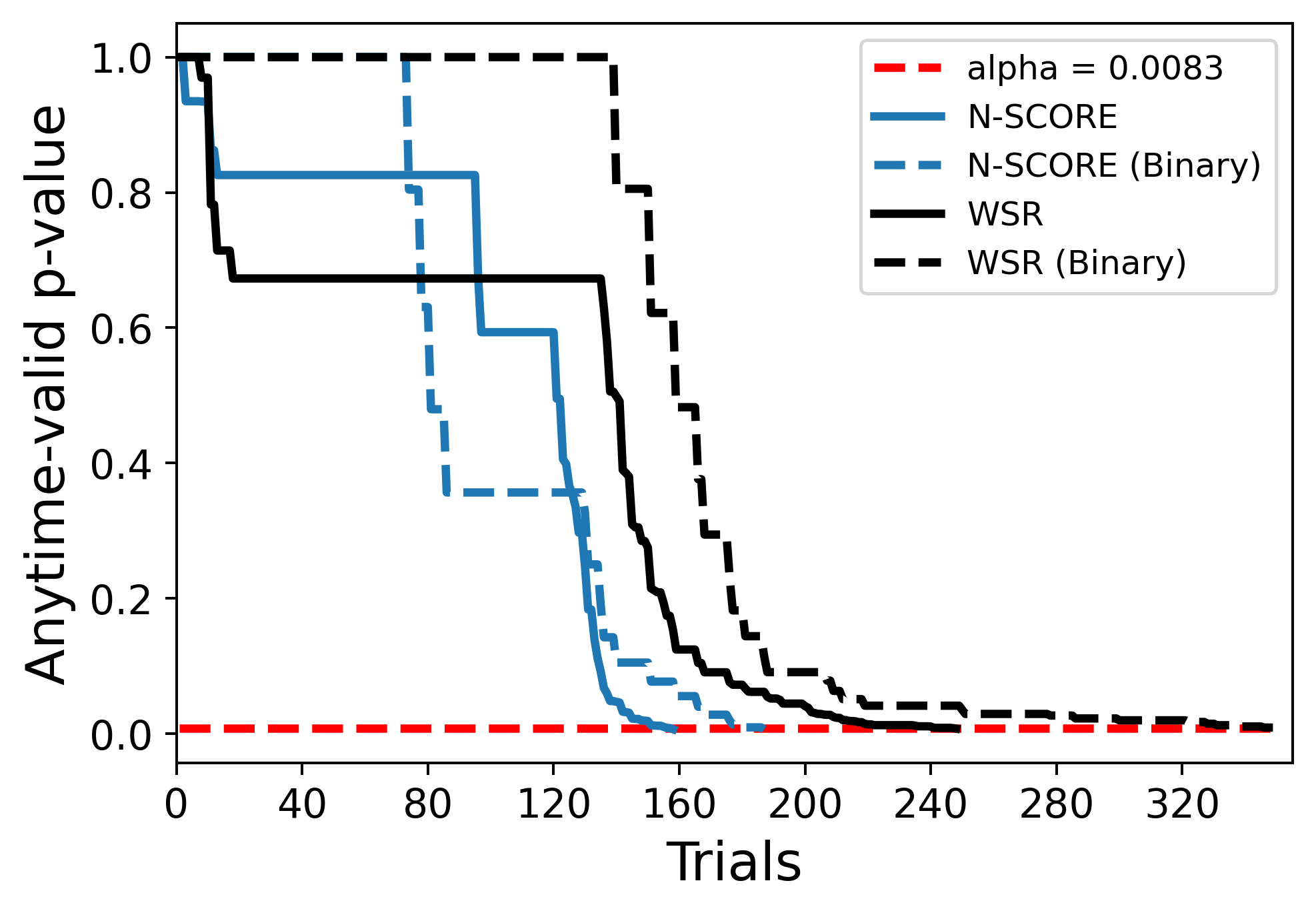}
        \caption{\(\pi_0\)-DROID vs. \(\pi_0\)-FAST-DROID}
        \label{fig:pi0_vs_pi0fast}
    \end{subfigure}\hfill 
    \begin{subfigure}{0.33\linewidth}
        \centering
        \includegraphics[width=\linewidth]{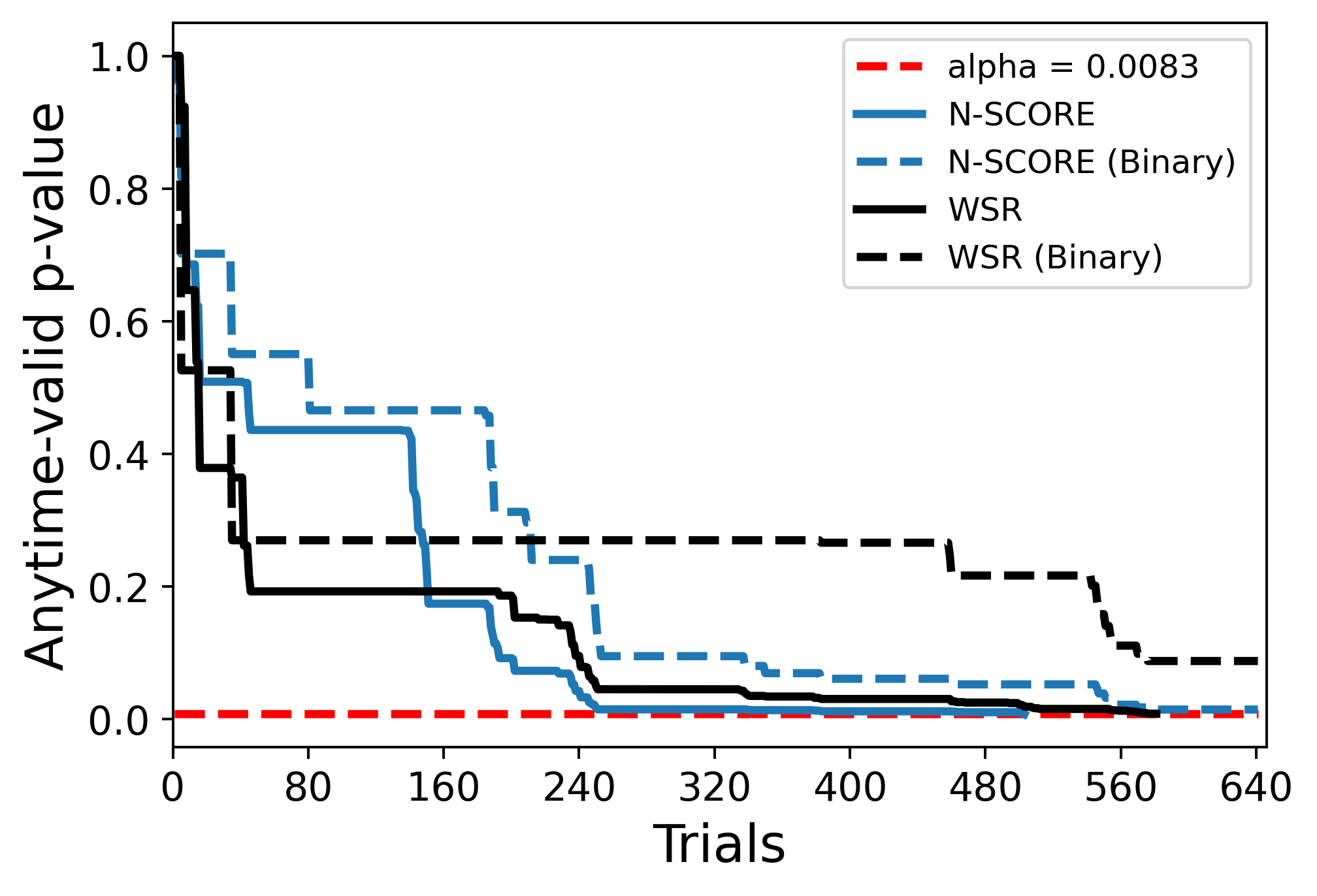}
        \caption{\(\pi_0\)-FAST-DROID vs. PG-Diff-DROID}
        \label{fig:pi0fast_vs_pgdiff}
    \end{subfigure}
\caption{\textbf{Anytime p-values of test procedures on multi-policy comparisons on RoboArena data with small gaps in empirical performance of the policies.} A test procedure terminates when its p-value is at or below the error threshold \(\alpha\). In this regime, binary metrics often fail to obtain a significant result at all, whereas methods that exploit more informative measures can stop significantly more quickly. }
\label{fig:pvalue}
\end{figure*}
\subsection{Data Independence in Robot Policy Evaluation}
Robot policy evaluations can be conducted in several ways, differing in how evaluation environments are sampled due to logistical feasibility. We clarify these distinctions for a better understanding of statistical assumptions and interpretation of evaluation results. In all of the following evaluation settings, we wish to estimate policy performance over some environment distribution \(\mathcal{D}_{\text{env}}\).

In the first setting, suppose the evaluator wishes to compare policies on a specific environment (i.e., a single initial condition), which is useful in discerning whether one policy succeeds while the other one fails more consistently in that particular setting. This corresponds to a degenerate environment distribution, in which we evaluate both policies on the same environment all the time. Despite the deterministic environment selection in this context (arising from the distribution degeneracy), our method remains valid because identical environments amount to i.i.d. samples in this special case.

In the second setting, the evaluator samples an environment i.i.d., and runs a single trial of every policy on that same environment before moving on to sample a new initial condition. This is useful in collecting pairwise preferences and is logistically easier for the evaluator to reset to the same environment for all policies\footnote{This is of course only theoretical: resets are inherently imprecise due to small errors in resetting the environment, small perturbations in camera poses, etc. This is discussed in more detail in, for example, \cite{kress2024robot}.}. This is the evaluation scheme adopted in RoboArena~\cite{atreya2025roboarena}.

In the third setting, the evaluator samples an environment i.i.d. with replacement for each trial and for each policy. For rich environment distributions, it is important to gather independent data from each trial for statistical inference. Specifically, this implies that the environment at  trial \(n\) for policy A might be different from that of trial \(n\) of policy B. While this ensures i.i.d. sampling, the frequencies of various initial conditions might not be equal across policies. 

A separate but important scheme is one that is exchangeable but not i.i.d. Often, this arises from stratification over an auxiliary variable -- for example, forcing the $N$ evaluations to be split equally so as to have exactly $N/L$ evaluations in each of $L$ task contexts. Stratification often induces implicit sampling without replacement to ensure that all policies see all $L$ `types' of initial conditions at exactly the same frequency, rather than the softer i.i.d. constraint of them having the same \emph{expected} frequency, but allowing the frequency to vary in finite samples.\footnote{RoboArena is again a special case, taking $L = N$.} Very concretely: if 100 evaluations are to consider a distribution that perturbs a target object or the background lighting color, stratification forces 50 evaluations \emph{exactly} for changes in target object and 50 evaluations \emph{exactly} for changes in background lighting. For truly i.i.d. data, this context itself is random, and we might observe, say, 55 of one and 45 of the other in a given evaluation sample. This exchangeable-but-not-i.i.d. scheme is adopted in~\citet{barreiros2025careful}. While slightly weaker than i.i.d., any potential effects due to the deviation from the i.i.d. assumption are practically negligible as long as all the policies are evaluated on each `type' of environment a sufficient number of times (i.e., when $L$ is small relative to $N$).

The importance of independence, and the challenge of $L \approx N$ with respect to the second setting, is clarified as follows. If there are latent or hidden correlations between different policies' performance levels with respect to the draw of the particular environment, then the variance of the difference in observed performance will be inflated\footnote{This can be intuitively seen if we imagine that the policies are negatively correlated, using the simple identity $\text{Var}(R_1 - R_0) = \text{Var}(R_1) + \text{Var}(R_0) - 2\text{Cov}(R_1, R_0)$. If the covariance is negative, the variance is amplified. If environments are chosen instead to be i.i.d. per trial \emph{and per policy}, then this covariance is precisely fixed to be $0$.}. Without accounting for this inflation, collecting results in the second setting could result in test procedures being overly optimistic, violating Type-1 Error control. When $L$ is smaller, it implies more independence in the $N/L$ evaluations per $L$ `types' or `contexts,' reducing the non-i.i.d. bias in practice. For small stratification in particular these effects are often small -- especially when the space of environments is very high-dimensional (i.e., covariances tend to diffuse in those spaces in practice). Therefore, for the purposes of running our evaluation algorithms, we do not attempt to correct for this phenomenon in the data, but we do note its presence. We hope that greater uptake of statistical uncertainty-aware evaluation procedures will also promote evaluation protocols that optimize for these considerations. 

To partially address the preceding point, we note a simple post-hoc correction to obtain truly i.i.d data from existing evaluations collected under the second setting. Recall that in this case, we have an ordered list of initial conditions (possibly with randomization and duplicates), and all the policies have been evaluated in this same order. For each initial condition in this list, we sample a policy at random and record its score, and ignore the evaluation outcomes of all other policies in the same round. Thus, if we have \(N\) trials (and corresponding environments) and \(K\) policies, we will have an effective sample size of \(N/K\), on average, for each policy. One can then apply test procedures such as~\OurMethod~and WSR on the reduced sample size of \(N/K\) evaluations. In essence, to robustly account for potential (worst-case) correlations in the realized environments, we must deflate the sample size by a factor of $K$, where $K$ is the number of policies being evaluated. However, this is generally inefficient as compared to simply resampling an environment configuration for each policy in turn, and should therefore only be undertaken \emph{post-hoc} (i.e., given that the data has already been collected in the aforementioned manner).

\end{document}